\PassOptionsToPackage{prologue,dvipsnames,svgnames,usenames}{xcolor}
\documentclass[11pt]{article}
\usepackage{fullpage}

\usepackage{amsmath,amssymb}
\usepackage{bm}
\usepackage[round]{natbib}

\usepackage{amssymb}
\usepackage{mathrsfs}
\usepackage{graphicx}
\usepackage{pdfpages}
\usepackage{afterpage}
\usepackage{caption}
\usepackage{rotating}
\usepackage{enumitem} 
\usepackage{algorithm2e}
\usepackage{svg}
\usepackage{stackengine}
\usepackage{float}
\usepackage{footmisc}
\RestyleAlgo{ruled}
\SetAlgorithmName{Procedure}{procedure}{List of Procedure}
\usepackage{tikz} 
\usetikzlibrary{matrix}
\usepackage{microtype}
\usepackage{preamble/star}
\usepackage[group-separator={,},group-minimum-digits={3}]{siunitx}

\usepackage{url}
\usepackage{graphicx}
\usepackage{subfigure}
\usepackage{caption}
\usepackage{multirow}

\usepackage{amsmath,amsthm,amssymb,bbm}
\usepackage{pifont}
\usepackage{color}

\usepackage[colorlinks,linkcolor=blue,anchorcolor=blue,citecolor=blue,urlcolor=black]{hyperref}
\newcommand*\samethanks[1][\value{footnote}]{\footnotemark[#1]}

\newcolumntype{P}[1]{>{\centering\arraybackslash}p{#1}}

\newcommand{\Rom}[1]{\text{\uppercase\expandafter{\romannumeral #1\relax}}}

\def\##1\#{\begin{align}#1\end{align}}
\def\$#1\${\begin{align*}#1\end{align*}}

\renewcommand{\max}{\mathop{\mathrm{max}}}

\usepackage{color}

\usepackage{amsmath,amsfonts,bm}

\def\eqref#1{equation~\ref{#1}}

\def\1{\bm{1}}

\DeclareMathAlphabet{\mathsfit}{\encodingdefault}{\sfdefault}{m}{sl}
\SetMathAlphabet{\mathsfit}{bold}{\encodingdefault}{\sfdefault}{bx}{n}

\newcommand{\E}{\mathbb{E}}

\newcommand{\R}{\mathbb{R}}

\begin{document}

\title{ \LARGE Neural Collapse Meets Differential Privacy:
Curious Behaviors of NoisyGD with Near-perfect Representation Learning}

\author{Chendi Wang\thanks{Equal contributions.} \thanks{University of Pennsylvania. Email: chendi@wharton.upenn.edu.}\quad Yuqing Zhu\samethanks[1] \thanks{University of California, Santa Barbara. Email: yuqingzhu@ucsb.edu.} \quad Weijie J.~Su\thanks{University of Pennsylvania. Email: suw@wharton.upenn.edu.}\quad Yu-Xiang Wang\thanks{University of California, San Diego. Email: yuxiangw@ucsd.edu.}}

\date{\today}

\maketitle

\vspace{-0.25in}

\begin{abstract}
A recent study by \citet{de2022unlocking} has reported that large-scale representation learning through pre-training on a public dataset significantly enhances differentially private (DP) learning in downstream tasks, despite the high dimensionality of the feature space.
To theoretically explain this phenomenon, we consider the setting of a layer-peeled model in representation learning, which results in interesting phenomena related to learned features in deep learning and transfer learning, known as Neural Collapse (NC).

Within the framework of NC, we establish an error bound indicating that the misclassification error is independent of dimension when the distance between actual features and the ideal ones is smaller than a threshold. Additionally, the quality of the features in the last layer is empirically evaluated under different pre-trained models within the framework of NC, showing that a more powerful transformer leads to a better feature representation. Furthermore, we reveal that DP fine-tuning is less robust compared to fine-tuning without DP, particularly in the presence of perturbations. These observations are supported by both theoretical analyses and experimental evaluation. Moreover, to enhance the robustness of DP fine-tuning, we suggest several strategies, such as feature normalization or employing dimension reduction methods like Principal Component Analysis (PCA). Empirically, we demonstrate a significant improvement in testing accuracy by conducting PCA on the last-layer features.
\end{abstract}

\section{Introduction}
Recently, privately fine-tuning a publicly pre-trained model with differential privacy (DP) has become the workhorse of private deep learning. For example, ~\citet{de2022unlocking} demonstrates that fine-tuning the last-layer of an ImageNet pre-trained Wide-ResNet achieves an accuracy of $95.4\%$ on CIFAR-10 with $(\epsilon=2.0, \delta = 10^{-5})$-DP, surpassing the $67.0\%$ accuracy from private training from scratch with a three-layer convolutional neural network~\citep{abadi2016DPSGD}. Additionally, \citet{li2021large, yu2021differentially} show that pre-trained  BERT~\citep{devlin2018bert} and GPT-2~\citep{radford2018improving} models achieve near no-privacy utility trade-off when fine-tuned for sentence classification and generation tasks.


However, the empirical success of privately fine-tuning pre-trained large models appears to defy the worst-case dimensionality dependence in private learning problems --- noisy stochastic gradient descent (NoisySGD) requires adding noise scaled to $\sqrt{p}$ to each coordinate of the gradient in a model with $p$ parameters, rendering it infeasible for large models with millions of parameters. This suggests that the benefits of pre-training may help mitigate the dimension dependency in NoisySGD. A recent work~\citep{li2022does} makes a first attempt on this problem --- they show that if gradient magnitudes projected onto subspaces decay rapidly, the empirical loss of NoisySGD becomes independent to the model dimension. 
However, the exact behaviors of gradients remain intractable to analyze theoretically, and it remains uncertain whether the ``dimension independence" property is robust across different fine-tuning applications.

In this work, we explore private fine-tuning behaviors from an alternative direction — we employ an representation of pre-trained models using the Neural Collapse (NC) theory \citep{MR4250189} and study the dimension dependence in a specific private fine-tuning setup --- fine-tuning only the last layer of the pre-trained model, a benchmark method in private fine-tuning.

\begin{figure}[ht]
  \centering
  {\includegraphics[width=0.8\textwidth]{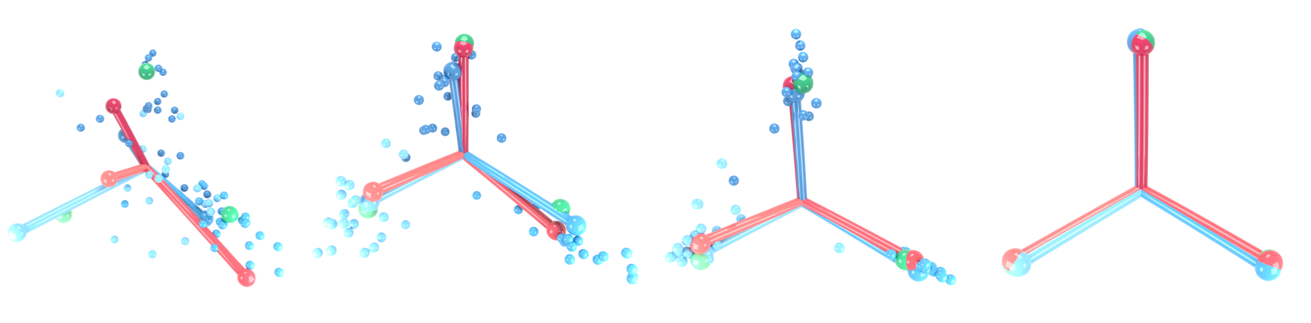}}
\caption{The figure depicts the evolution of the feature layer outputs of a VGG-13 neural network when trained on CIFAR-10 with three randomly selected classes. Each class is represented by a distinct color in the small blue sphere. As the training progresses, the last-layer feature means collapse onto their corresponding classes. Credit to ~\citet{MR4250189}.}
  \label{fig:nc_example}
\end{figure}

To elaborate, Neural Collapse~\citep{MR4250189, fang2021exploring,doi:10.1073/pnas.2221704120} is a phenomenon observed in the late stage of training deep classifier neural networks. At a high level, NC 
suggests that when a deep neural network is trained on a $K$-class classification task, the last-layer feature converges to $K$ distinct points (see Figure~\ref{fig:nc_example} for $K=3$). This insight provides a novel characterization of deep model behaviors, sparking numerous subsequent studies that delve into this phenomenon~\citep{masarczyk2023tunnel, li2023what}.
Notably, \citet{masarczyk2023tunnel} shows that when pre-training a model on CIFAR-10 and fine-tuning it on a subset of 10 classes from CIFAR-100, the collapse phenomenon has been observed.


To summarize,  we aim to  answer the following questions about private learning under Neural Collapse.
\begin{enumerate}
    \item \emph{When fine-tuning the last layer, what specific property of the features of that layer is important to make private learning dimension-independent?}
    \item \emph{How robust is 
    this dimension-independence property against perturbations?}
\end{enumerate}

\subsection{Contributions of this paper}

In essence, our contribution lies in correlating the performance and robustness of privately fine-tuning with the recently proposed Neural Collapse theory.



Theoretically, we identify a key structure of the last-layer feature  that leads to a
 dimension-independent misclassification error (using 0-1 loss) in noisy gradient descent (NoisyGD). We formalize this structure through the feature shift parameter $\beta$, which measures the $\ell_{\infty}$-distance between the actual last-layer features obtained from the pre-trained model on a private set, and the ideal maximally separable features posited by Neural Collapse theory. A smaller value of $\beta$ indicates a better representation of the last-layer feature.
We show that if the feature shift parameter $\beta$ remains below a certain threshold related to the model dimension $p$,  the sample complexity bound of achieving a $\gamma$ misclassification error is dimension-independent. 

Empirically, we evaluate the feature shift vector $\beta$ when fine-tuning different transformers on CIFAR-10. Figure \ref{fig:nc_cifar10} plots the distribution of per-class $\beta$ when the pre-trained transformer is either the Vision Transformer (ViT) or ResNet-50. Blue and green scatter plots represent the $\beta$ values for each sample from the two models, while purple and yellow scatter plots denote the median $\beta$ within each class. The median $\beta$ is centered around $0.10$ for ViT and around $0.20$ for ResNet-50.
The pre-trained ViT model is known to have better feature representations than the ResNet-50 model.
Our results show that: (1) $\beta$ is bounded for the two pre-trained models, with very few outliers, and (2) the better the pre-trained model, the smaller the $\beta$. For example, ViT (with $\beta\approx 0.1$) outperforms ResNet-50 (with $\beta\approx 0.2$) due to its smaller shift parameter.
We postpone the details of our experiments to Section \ref{sec:empirical-feature-shift}.

\begin{figure}[ht]
  \centering
  \hspace{1cm}
  \includegraphics[width=0.8\textwidth]{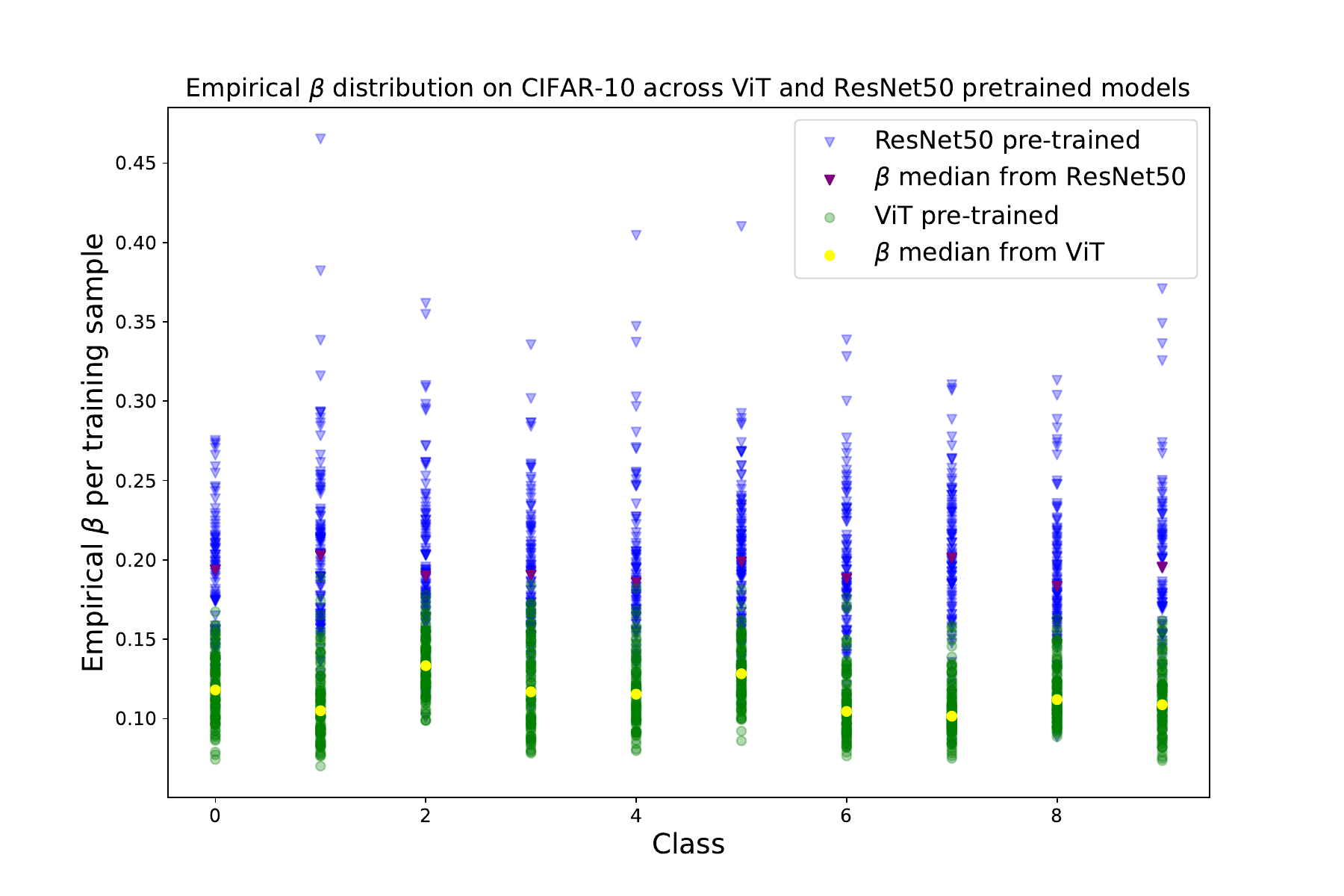}
  \caption{\textbf{CIFAR-10.} A figure depicting the feature shift parameter $\beta$ when fine-tuning different pre-trained models on CIFAR-10. As observed, ViT performs better than ResNet-50, as the shift parameter is much smaller. The feature shift vectors are quite stochastic.}
  \label{fig:nc_cifar10}
  \vspace{-1em}
\end{figure}

Moreover, we study the robustness of the  ``dimension-independence" property against various types of perturbations, including stochastic, adversarial, and offset perturbations. Each perturbation type will alter the feature presentations, increasing $\beta$ and potentially making the sample complexity bound 
dimension-dependent.  Notably, we find that the adversarial perturbations impose a stricter limitation on $\beta$, lowering the acceptable upper-bound threshold for $\beta$ from $p^{-\frac{1}{2}}$ to $p^{-1}$,  meaning that adversarial perturbations are more fragile compared to other types of perturbations.

Moving forward,  to mitigate the  ``non-robustness of dimension-independency" issue under perturbations, we propose solutions to enhance the robustness of NoisyGD. We find out that releasing the mean of feature embeddings effectively neutralizes offset perturbation (detailed in Section \ref{sec:normalization}), and dimension reduction techniques like PCA (detailed in Section \ref{sec: pca}) reduce the constraint on the feature shift parameter $\beta$, thus improve NoisyGD's robustness. Our theoretical results shed light on the recent success of private fine-tuning. Specifically, our analysis on PCA provides a plausible explanation for the effectiveness of employing DP-PCA in the first private deep learning work~\citep{abadi2016DPSGD}.

We summarize our contributions below:

 \begin{itemize}

 \item We present a direct analysis of the misclassification error (on population) under 0-1 loss.
 Existing analyses of the excess risk on population \citep{DBLP:conf/focs/BassilyST14, DBLP:conf/nips/BassilyFTT19,DBLP:conf/nips/BassilyFGT20,DBLP:conf/stoc/FeldmanKT20} mainly focus on the excess risk under convex surrogate loss, leading to sample complexity bounds (inverse) polynomial in the required excess risk.
 We show that Neural Collapse theory allows us to directly
 bound 0-1 loss, which results in a logarithmic sample complexity bound in the misclassification error $\gamma.$
 \item We introduce a feature shift parameter $\beta$ to measure the discrepancy between actual and ideal last-layer features. Our theoretical findings show that this feature structure is crucial to guarantee NoisyGD's resistance to dimension dependence. In particular, when $\beta \leq p^{-\frac{1}{2}}$, the sample complexity for NoisyGD becomes $O(\log(1/\gamma))$, which is dimension-independent. We also provide an empirical evaluation of this feature shift parameter $\beta$ on CIFAR-10 using ImageNet pre-trained vision transformer and ResNet-50, which shows that the vision transformer outperforms ResNet-50 as it leads to a smaller $\beta$.
 
 \item We theoretically analyze the misclassification error and robustness of NoisyGD against several types of perturbations and class-imbalanced scenarios, with sample complexity bounds summarized in Table~\ref{tab:complexity}.
 
\item We empirically validate our theoretical findings on private fine-tuning vision models. We show that privately fine-tuning an ImageNet pre-trained vision transformer is not affected by the last-layer dimension on CIFAR-10. However, we observe a degradation in the utility-dimension trade-off when a minor perturbation is  introduced, aligning with our theoretical results.
\item We propose two solutions to address the non-robustness challenges in DP fine-tuning with theoretical insights. In particular,  we empirically apply PCA on the perturbed last-layer feature before private fine-tuning,  demonstrating that the utility of NoisySGD remains  unaffected by the original feature dimension. We believe these results will provide new insights into enhancing the robustness of  private fine-tuning.

\end{itemize}

\begin{table}[htb]
\centering
\begin{tabular}{l|c|c}
\toprule
Setting & Assumption & Sample complexity\\
\hline
\hline
\multirow{3}{2cm}{Non-private learning with GD} & Perfect NC or $\beta$-approximate NC with $\beta^2p\leq 1$&  No.\ of classes    \\
                               & $\beta$-approximate NC & $p\beta^2 \log\frac{1}{\gamma}$   \\
                               & $\beta$-approximate NC (separability) & $p\beta^4 \log\frac{1}{\gamma}$    \\ 
            \hline
\multirow{7}{2cm}{Private learning ($\rho$-zCDP) with NoisyGD} & Perfect NC&  $\frac{2\sqrt{\log(1/\gamma)}}{\sqrt{\rho}}$ \\
                               & $\beta$-approximate NC & $p\beta^2 \log\frac{1}{\gamma} + \frac{\max\{p\beta^2, 1\} \sqrt{\log(1/\gamma)}}{\sqrt{2\rho}}$   \\
                               & $\beta$-approximate NC (separability) & $p\beta^4 \log\frac{1}{\gamma} + \frac{\max\{p\beta^2, 1\} \sqrt{\log(1/\gamma)}}{\sqrt{2\rho}}$   \\ 
                &$\widetilde{\beta}$-stochastic perturbation (test)& $\frac{\max\{\sqrt{p}\widetilde{\beta}, 1\} \sqrt{\log(1/\gamma)}}{\sqrt{2\rho}}$  \\
		&$\widetilde{\beta}$-adversarial perturbation (test) & $\frac{\max\{ p\widetilde{\beta}, 1\} \sqrt{\log(1/\gamma)}}{\sqrt{2\rho}}$ \\
	  & $\widetilde{\beta}$-offset perturbation (training)  &  $\frac{\max\{\sqrt{p}\widetilde{\beta}, 1\} \sqrt{\log(1/\gamma)}}{\sqrt{2\rho}}$  \\
	 & $\widetilde{\beta}$-offset perturbation + $\alpha$-Class imbalance &$\frac{\max\{\sqrt{p}\widetilde{\beta}, 1\} \sqrt{\log(1/\gamma)}}{(1 - \widetilde{\beta} + 2\widetilde{\beta}\alpha)\sqrt{2\rho}}$\\
            \bottomrule
\end{tabular}
\caption{Summary of the sample complexity of achieving a misclassification error $\gamma$ of private learning under $\rho$-zCDP.
We consider perfect features, actual features (GD and NoisyGD), and perturbed features (stochastic, adversarial, offset perturbations).
For the actual features, we assume that the feature shift vectors of traing and testing features are from the same distribution.
The separability assumption here is that all $p$ components of the feature shift vectors are independent.
We have also considered the effects of $\alpha$-class imbalance.}
 
\label{tab:complexity}
\end{table}

\section{Preliminaries and Problem Setup}
\label{sec:preliminary}

In this section, we review the private fine-tuning literature, introduce the Neural Collapse phenomenon, and formally describe the private fine-tuning problem under (approximate) Neural Collapse.

\paragraph{Symbols and notations.} Let the data space be $\cZ$. A dataset $\cD$ is a set of individual data points $\{z_1,z_2,...\}$ where $z_i\in\cZ$. Unless otherwise specified, the size of the data $|\cD|:=n$.
$\cZ =  \cX\times \cY$ where $\cX$ is the feature space and $\cY$ the label space with $\cX\subset \R^p$ and $\cY = \{1,2,...,K\}$. A data point $z_i$ is a feature-label pair $(x_i,y_i)$. Occasionally, we overload $y_i$ to also denote the one-hot representation of the label. We use standard probability notations, e.g., $\Pr[\cdot]$ and $\E[\cdot]$ for probabilities and expectations.  Other notations will be introduced when they first appear.

\paragraph{Differentially private learning.} The general setting of interest is \emph{differentially private learning} where the goal is to train a classifier while satisfying a mathematical definition of privacy known as differential privacy \citep{dwork2006differential}.  DP ensures that any individual training data point cannot be identified using the trained model and any additional side information. 

More formally, we adopt the popular modern variant called zero-centered Concentrated Differential Privacy (zCDP), as defined below.


\begin{definition}[Zero-Concentrated Differential Privacy, zCDP, \citet{DBLP:conf/tcc/BunS16}]
Two datasets $\cD_0, \cD_1$ are neighbors if they can be constructed from each other by adding or removing one data point. A randomized mechanism $\cA$ satisfies $\rho$-zero-concentrated differentially private ($\rho$-zCDP) if, for all neighboring datasets $\cD_0$ and $\cD_1$, we have
$
    R_{\alpha}(\mathcal{A}(\cD_0)\|\cA(\cD_1)) \leq \rho \alpha,
$
where $  R_{\alpha}(P\|Q) = \frac{1}{\alpha - 1} \log \int \left(\frac{p(x)}{q(x)} \right)^{\alpha} q(x) dx$ is the R\'enyi divergence between two distributions $P$ and $Q$.
\end{definition}
In the above definition, $\rho\geq 0$ is the privacy loss parameter that measures the strength of the protection. $\rho=0$ indicates perfect privacy,  $\rho = \infty$ means no protection at all. The privacy protection is considered sufficiently strong in practice if  $\rho$ is a small constant, e.g., $1,2,4,8$.   For readers familiar with the standard approximate DP but not zCDP, $\rho$-zCDP implies  $(\epsilon,\delta)$-DP for all $\delta >0$ with $\epsilon = \rho + 2\sqrt{\rho\log(1/\delta)}$.

The goal of differentially private learning is to come up with a differentially private ($\rho$-zCDP) algorithms that outputs a classifier $f:\cX \rightarrow \cY$ such that the misclassification error
$$
\mathrm{Err}(f) = \E_{(x,y)\sim \cP} \left[\mathbf{1}(f(x)\neq y)\right]
$$
is minimized (in expectation or with high probability), where $\cP$ is the data distribution under which the training data is sampled from i.i.d.

For reasons that will become clear soon, we will focus on linear classifiers parameterized by $W\in\R^{K\times p}$  of the form
$$f_W(x) = \argmax_{y\in[K]} [W x]_y.$$

\paragraph{Noisy gradient descent.}
Noisy Gradient Descent or its stochastic version Noisy Stochastic Gradient Descent  \citep{song2013stochastic,abadi2016DPSGD} is a fundamental algorithm in DP deep learning. 
To minimize the loss function $\cL(\theta) := \sum_{i=1}^n \ell(\theta, z_i)$, the NoisyGD algorithm updates the model parameter $\theta_t$ by combining the gradient with an isotropic Gaussian.
\begin{align}
\label{eq:DP-GD}
 \theta_{t+1} = \theta_t - \eta_t \left(\sum_{i=1}^n \nabla \ell(\theta_t, z_i) + \cN\left(0, \frac{G^2}{2\rho} I_p\right) \right).
 \end{align}
Here $G$ is the $\ell_2$-sensitivity of the gradient, and the algorithm runs for $T$ iterations that satisfy $T\rho$-zCDP.

However, the excess risk on population of NoisySGD must grow as $\sqrt{p}/\epsilon$ \citep{DBLP:conf/focs/BassilyST14,DBLP:conf/nips/BassilyFTT19}, which limits private deep learning benefit from model scales. To overcome this, DP fine-tuning \citep{de2022unlocking, li2021large, bu2022differentially} is emerging as a promising approach to train large models with privacy guarantee.

\paragraph{Private fine-tuning.}
In DP fine-tuning, we begin by pre-training a model on a public dataset and then privately fine-tuning the pre-trained model on the private dataset. Our focus is on fine-tuning the last layer of pre-trained models using the NoisySGD/NoisyGD algorithm, which has consistently achieved state-of-the-art results across both vision and language classification tasks~\citep{de2022unlocking, tramer2020differentially, bao2023dpmix}.
A recent study \citep{DBLP:conf/nips/TangPSM23} shows that DP fine-tuning only the last layer can achieve higher accuracy than fine-tuning all layers.
However, we acknowledge that in some scenarios, fine-tuning all layers under DP can result in better performance, as demonstrated in the CIFAR-10 task by \citet{de2022unlocking}. The comprehensive analysis of dimension-dependence in other private fine-tuning benchmarks remains an area for future investigation.

\paragraph{Theoretical setup for private fine-tuning.}
For a $K$-class classification task, we rewrite each data point $z$ as $z= (x,y)$ with $x\in\mathbb{R}^p$ being the feature and $y = (y_1,\cdots, y_K)\in\{0,1\}^K$ being the corresponding label generated by the one-hot encoding, that is $y$ belongs to the $k$-th class if $y_k = 1$ and $y_j = 0$ for $j\neq k.$

When applying NoisyGD for fine-tuning the last-layer parameters, the model is in a linear form. 
Thus, we consider the linear model $f_{W}(x) = W x$ with $W\in\mathbb{R}^{K\times p}$ being the last-layer parameter to be  trained and $x$ is the last layer feature of a data point.
The parameter $\theta_t$ in NoisyGD is the vectorization of $W$.
Let $\ell:\mathbb{R}^K \times \mathbb{R}^K\rightarrow \mathbb{R}$ be a loss function that maps $f_{W}(x) \in \mathbb{R}^K$ and the label $y$ to $\ell(f_W(x), y).$
For example, for the cross-entropy loss, we have $\ell(f_W(x), y) = -\sum_{i=1}^K y_i\log[(f_W(x))_i].$

\paragraph{Misclassification error.}
For an output $\widehat{W}$ and a testing data point $(x,y)$, the misclassification error we considered is defined as $\Pr\left[y\neq f_{\widehat{W}}(x)\right],$ where the probability is taken with respect to the randomness of $\widehat{W}$ and $(x,y)\sim\mathcal{P}.$

\paragraph{Beyond the distribution-free theory.} 
Distribution-free learning with differential privacy is however known to be statistically intractable even for linear classification in $1$-dimension \citep{chaudhuri2011sample,bun2015differentially,MR3567451}.  Existing work resorts to either proving results about (convex and Lipschitz) surrogate losses \citep{DBLP:conf/focs/BassilyST14} or making assumptions on the data distribution \citep{chaudhuri2011sample,DBLP:conf/focs/BunLM20}. 
For example, \citet{chaudhuri2011sample} assumes bounded density, and \citet{DBLP:conf/focs/BunLM20} shows that linear classifiers are privately learnable if the distribution satisfies a large-margin condition.
Our setting, as detailed in Section \ref{sec:random-perturb-both}, can be viewed as a \emph{new family of distributional assumptions} motivated by the recent discovery of the Neural Collapse phenomenon. As we will see, these assumptions not only make private learning statistically and computationally tractable (using NoisyGD), but also produce sample complexity bounds that are \emph{dimension-free} and \emph{exponentially faster} than existing results that are applicable to our setting.

\begin{definition}[Sample complexity for private $\mathfrak{D}$-learnability]
\label{defn:sample-complexity}
For a set of distributions $\mathfrak{D}$, the sample complexity of private $\mathfrak{D}$-learning a hypothesis class $\mathcal{H}$ under $\rho$-zCDP is defined to be
$$
n_{\mathcal{H},\mathfrak{D},\rho}(\gamma) = \min\left\{ n\in \mathbb{N}  \;\middle| \;\inf_{\mathcal{A} \text{ satisfies }\rho-\text{zCDP}} \sup_{\cP\in\mathfrak{D}}  \E_{\mathcal{A}, \text{data}\sim \cP^n}\left[ \mathrm{Err}_{\cP}(\cA(\text{data})) - \inf_{h\in\mathcal{H}} \mathrm{Err}_{\cP}(h)\right] \leq  \gamma\right\}
$$
i.e., the smallest integer such that the minimax excess risk is smaller than $\gamma$.
\end{definition}
\noindent
We say $\mathfrak{D}$ is learnable under $\rho$-zCDP if $n_{\mathcal{H},\mathfrak{D},\rho}$ is finite, meaning there exists an algorithm $\mathcal{A}$ such that the expected excess risk is smaller than $\gamma$ with finitely many training data.
It is obvious that the misclassification error defined above can be written as $\E_{\mathcal{A}, \text{data}\sim \cP^n}\left[ \mathrm{Err}_{\cP}(\cA(\text{data}))\right]$ as the randomness of $\widehat{W}$ comes from the randomized algorithm $\mathcal{A}$ and the training data.
Since the second term $\inf_{h\in\mathcal{H}} \mathrm{Err}_{\cP}(h)$, which is the misclassification error of the optimal classifier, is non-negative, the misclassification error is always larger than the excess risk. Thus, in Section \ref{sec:theory}, we investigate the convergence of the misclassification error, which implies an upper bound on the excess risk as well.
We will focus on linear classifiers, due to fine-tuning the last layer, and study distribution families $\mathfrak{D}$ inspired by Neural Collapse.


\paragraph{Neural Collapse.}
Neural Collapse~\citep{MR4250189, fang2021exploring,doi:10.1073/pnas.2221704120} describes a phenomenon about the last-layer feature structure obtained when a deep classifier neural network converges. It  demonstrates that the last-layer feature converges to the column of an equiangular tight frame (ETF). 
Mathematically, an ETF is a matrix 
\begin{align}
\label{eq:defn-etf}
    M &= \sqrt{\frac{K}{K-1}} P \left(I_K - \frac{1}{K}\mathbf{1}_K \mathbf{1}_K^T\right) \in\mathbb{R}^{p\times K},
\end{align}
where $P=[P_1,\cdots, P_K]\in\mathbb{R}^{p\times K}$ is a partial orthogonal matrix such that $P^T P = I_K$. For a given dimension $d=p$ or $K$, we denote $I_d\in\mathbb{R}^d$ the identity matrix and denote $\mathbf{1}_d=[1,\cdots, 1]^T\in\mathbb{R}^d$.
Rewrite $M = [M_1,\cdots, M_K]$ with $M_k$ being the $k$-th column of $M$, that is,  the ideal feature of the data belonging to class $k$.


We adopt the Neural Collapse theory to describe an ideal feature representation of applying the pre-trained model on the private set.  However, achieving perfect collapse on the private set is an ambitious assumption, as in practice, the private feature of a given class is distributed around $M_k$. Therefore, we introduce a feature shift parameter $\beta$ to measure the discrepancy between the actual feature and the perfect feature $M_k$.


\begin{definition}[Feature shift parameter $\beta$] For any $1\leq k \leq K$, given a feature $x$ belonging to the $k$-th class and the perfect feature $M_k$,  we define $\beta =  ||x- M_k ||_{\infty}$ as the feature shift parameter of $x$ that measures the $\ell_{\infty}$ distance between $x$ and $M_k$.
\end{definition}

Here, we use the $\ell_\infty$ norm since it is related to adversarial attacks, which are important in our study of the robustness of NoisyGD. Our numerical results in Figure \ref{fig:nc_cifar10} show that $\beta$ is bounded on CIFAR-10 if the pretrained model is the vision transformer or ResNet-50. 

\section{Bounds on misclassification errors and robustness in  private fine-tuning}
\label{sec:theory}

In this section, we establish  bounds on the misclassification error for both GD and the NoisyGD . 

Section \ref{sec:fixedperturb} aims to delineate the connection between the feature shift parameter $\beta$ and the misclassification error. Additionally, we derive a threshold for $\beta$ below which the misclassification error is dimension-independent.

In Section \ref{sec:robustness}, our focus is the robustness of private fine-tuning. Specifically, we elucidate how various perturbations impact both $\beta$ and the misclassification error.

\subsection{Bounds on misclassification errors}
\label{sec:fixedperturb}
We consider a binary classification problem with a training set $\{(x_i, y_i)\}_{i=1}^n$, where $x_i$ represents features and $y_i \in \{\pm 1\}$ are the labels.
For the broader multi-class scenarios, we state our theory for the perfect case in Section \ref{sec:perfect-collapse}.
The rest theory can be extended to the multi-class case similarly.

For binary classification problems, the ETF $M = [M_1,M_2]$ satisfies $M_1 = -M_2$, which is equivalent to $M_1 = e_1$ and $M_2 = -e_1$ up to some rotation map. In fact, for this binary classification problem, we have $M_1 = \frac{P_1 - P_2}{\sqrt{2}}$, according to the definition of an ETF in Eq.\ \ref{eq:defn-etf}. Similarly, it holds that $M_2 = \frac{P_2 - P_1}{\sqrt{2}}$. Thus, we have $M_1 = -M_2$ and $\|M_1\|_2 = \|M_2\|_2 = 1$. Since our theory depends only on the norm of the feature means $M_1$ and $M_2$, and the angle between $M_1$ and $M_2$, while the rotation matrix $P$ will not change the misclassification error, without loss of generality, we assume that $M_1=e_1$ and $M_2=-e_1$. For a data point $(x,y)$ with $y=1$, recall the feature shift parameter $\beta = \|x - e_1\|_{\infty}$, which is the infinity norm of $v = x - e_1$. We call $v$ a feature shift vector since it is the difference between an actual feature and the perfect one. Similarly, if $y = -1$, the feature shift vector is $v = x + e_1$. For a training set $\{(x_i,y_i)\}_{i=1}^n$, let $\{v_i = x_i - \mathrm{perfect \, feature}\}_{i=1}^n$ be a sequence of feature shift vectors.

We first consider the scenario where the shift vectors $v_i$'s are i.i.d.\ copies of a symmetric centered random vector $v$ with $\|v\|_\infty \leq \beta$ in Section \ref{sec:random-perturb-both}.
This setting is quite practical as can be seen from Figure \ref{fig:nc_example}.
Notably, the law of large numbers ensures that $\frac{1}{n}\sum_{i=1}^n v_i$ converges to zero. Consequently, the mean of the features approaches the ideal feature, in line with the principles of the Neural Collapse theory.
Moreover, in our theory for the softmargin case with $\beta^2p >1$, we need further assumption to make the model linearly separable. We assume that all $p$ components of $v$ are independent of each other.

Furthermore, acknowledging that at times, the feature may be influenced by a fixed vector, such as an offset shift vector which will not change the margin and angles between features, we also investigated the case where $v_i$ is deterministic in Section \ref{sec:fix-perturb-both}.


\subsubsection{Stochastic shift vectors}
\label{sec:random-perturb-both}

For conciseness, we mainly focus on the results for 1-iteration NoisyGD, which is sufficient to ensure the convergence.
As presented in Theorem \ref{thm:multi-projected-GD}, our theory can be extended to multi-iteration projected NoisyGD.
However, the dimension dependency  can not be mitigated using multiple iterations.
The proofs of all results in this section are given in Appendix \ref{sec:proof-perturb}.

For 1-iteration GD without DP guarantee, the output is 
$
    \widehat{\theta}_{\mathrm{GD}} = \eta \sum_{i=1}^n y_i x_i.
$
Moreover, the 1-iteration NoisyGD outputs 
$  \widehat{\theta}_{\mathrm{NoisyGD}} = \eta \sum_{i=1}^n y_ix_i + \mathcal{N}(0,\sigma^2). 
$
For the private learning problem,  
the sensitivity of the gradient is $G = \sup_{x_i}\|x_i\|_2 = \sqrt{1 + \beta^2 p},$ which is dimension dependent.
If we still want $G$ to be dimension independent, then every data point needs to be shrunk to $(e_1 + v) / \sqrt{1 + \|v\|_2^2}.$ In both cases, the error bounds remain the same.
In the testing procedure, we consider a testing point $(x,y)\sim\cP.$

For an estimate $\widehat{\theta}$ whose randomness is from a training dataset drawn independently from $\mathcal{P}$ and a randomized algorithm $\mathcal{A}$, the misclassification error $\E_{\mathcal{A},\text{data}\sim \cP^n}\left[ \mathrm{Err}_{\cP}(\cA(\text{data}))\right]$ in Definition \ref{defn:sample-complexity} can be rewritten as
\begin{align*}
    \E_{\mathcal{A},\text{data}\sim \cP^n}\left[ \mathrm{Err}_{\cP}(\cA(\text{data}))\right] = \E_{\widehat{\theta}}\E_{(x,y)\sim \mathcal{P}}\left[\mathbf{1}(y\neq \mathrm{sign}(\widehat{\theta}^T x))\right] = \Pr\left[y \widehat{\theta}^T x <0\right],
\end{align*}
where the probability is taken with respect to the randomness of both $(x,y)\sim\cP$ and $\widehat{\theta}$.

\begin{theorem}[misclassification error for GD]
\label{thm:GD-random perturbation}
   Let $\widehat{\theta}_{\mathrm{GD}}$ be a predictor trained by GD under the cross entropy loss with zero initialization.
Then, we have the following error bound on the misclassification error.
\begin{itemize}
    \item   If we assume that $\beta^2p \leq 1,$ then it holds
$
    \Pr[y\widehat{\theta}_{\mathrm{GD}}^T x < 0]  =0
$
for $n$ greater than the number of classes.
As a result, to achieve a misclassification error $\gamma$, the sample complexity is constant.
\item 
In general, it holds
   \begin{align*}
    \Pr[y\widehat{\theta}_{\mathrm{GD}}^T x < 0]  \leq \exp\left(-\frac{n}{2\left(\beta^4p^2 + \frac{1}{3}\beta^2p\right)}\right).
\end{align*}
Therefore, to achieve a misclassification error  $\gamma$, the sample complexity is $O\left(p\beta^2\log(1/\gamma)\right).$ 
If we further assume that all $p$ components of $v$ are independent of each other, then, it holds
\begin{align*}
    \Pr[y\widehat{\theta}_{\mathrm{GD}}^T x < 0]  \leq \exp\left(-\frac{n}{2\left(\beta^4p + \frac{1}{3}\beta^2\right)}\right).
\end{align*}
Thus, to achieve a misclassification error $\gamma$, the sample complexity is $O\left(p\beta^4\log(1/\gamma)\right).$ 
\end{itemize}
\end{theorem}



\begin{theorem}[misclassification error for NoisyGD]
\label{thm:NoisyGD-random perturbation}
   Let $\widehat{\theta}_{\mathrm{NoisyGD}}$ be a predictor trained by NoisyGD under the cross entropy loss with zero initialization.
   Then, we have the following error bound on the misclassification error.
\begin{align}
\label{eq:Noisy-bound}
\Pr\left[y\widehat{\theta}_{\mathrm{NoisyGD}}^T x < 0\right] \leq  \exp\left(-\frac{n^2\rho}{2(1 + \beta^2p)^2} \right) + \exp\left(-\frac{n}{8\left(\beta^4p^2 + \frac{1}{3}\beta^2p\right)}\right).
\end{align}
As a result, to achieve a misclassification error $\gamma$, the sample complexity is $O\left(\frac{(1 + \beta^2p)^2\sqrt{\log\frac{1}{\gamma}}}{2\rho} + p\beta^2\log(1/\gamma) \right).$ 
If we further assume that all $p$ components of $v$ are independent of each other, then, it holds
\begin{align*}
   &\Pr\left[y\widehat{\theta}_{\mathrm{NoisyGD}}^T x < 0\right] \leq  \exp\left(-\frac{n^2\rho}{2(1 + \beta^2p)^2} \right) + \exp\left(-\frac{n}{8\left(\beta^4p + \frac{1}{3}\beta^2\right)}\right).
\end{align*}
To achieve a misclassification error $\gamma$, the sample complexity is 
$
    O\left( \frac{(1 + \beta^2p)^2\sqrt{\log\frac{1}{\gamma}}}{2\rho} + 4p\beta^4\log\frac{1}{\gamma} \right).$
\end{theorem}

\paragraph{Remark.}
Note that with further assumptions on feature separability, the second term in Equation \ref{eq:Noisy-bound} (which aligns with GD in Theorem \ref{thm:GD-random perturbation}) can be improved from $\beta^2p$ to $\beta^4p$. However, the first term, caused by DP,  remains unchanged by this assumption. Thus, NoisyGD has a stricter requirement on feature quality due to the added random noise.
Theorems \ref{thm:GD-random perturbation} and \ref{thm:NoisyGD-random perturbation} indicate that the error bound is exponentially close to 0 under the following conditions:  $\beta\leq p^{-1/2}$ for both NoisyGD and GD and $\beta \leq p^{-1/4}$ for GD under stronger assumptions. This result is dimension-independent when $\beta$ satisfies the above conditions.
Moreover,  GD has robustness against larger shift vectors compared to NoisyGD.
This aligns with the observations from our experiments detailed in Section \ref{sec:experiments}, where we note a significant decrease in accuracy with increasing dimensionality.
In addition, when $\beta \leq p^{-1/2}$, the misclassification error for GD is always $0$ while that for NoisyGD is $\exp\left( -\frac{n\rho}{1 + \beta^2p}\right)$.

\paragraph{Promising properties for perfect collapse.}
In the special case $\beta=0,$ all features are equivalent to the perfect feature. For this perfect scenario, numerous promising properties are outlined as follows. The details are discussed in Section \ref{sec:perfect-collapse}.

\begin{enumerate}
	\item The error bound is exponentially close to $0$ if $\rho\gg G^2/n^2$ --- very strong privacy and very strong utility at the same time.
\item The result is dimension independent --- it doesn't depend on the dimension $p$.
\item The result is robust to class imbalance for binary classification tasks.
	\item The result is independent of the shape of the loss functions. Logistic loss works, while square losses also works. 
	\item The result does not require careful choice of learning rate. Any learning rate works equally well.
\end{enumerate}

\noindent\paragraph{Neural collapse in domain adaptation:} 
In many private fine-tuning scenarios, the model is initially pre-trained on an extensive dataset with thousands of classes (e.g., ImageNet), and is subsequently fine-tuned for a downstream task with a smaller number of classes. Our theory for the perfect case can be extended to the domain adaptation context, as detailed in Appendix \ref{proof:error-zero-
init-small}.

\subsubsection{Deterministic shift vectors}
\label{sec:fix-perturb-both}
In this section, we consider the case where each $v_i$ is a fixed vector with $\|v_i\|_\infty \leq \beta.$
Recall that the 1-iteration NoisyGD outputs 
$  \widehat{\theta}_{\mathrm{NoisyGD}} =\eta \left(n e_1 + \sum_{i=1}^n v_i\right) + \mathcal{N}(0,\sigma^2).
$
As discussed, when the feature is deterministic without assumptions on the distribution, the dataset is linearly separable only when $\beta^2p$ is less than 1.


\begin{theorem}[misclassification error for NoisyGD]
\label{thm:fixed-perturbation}
   Let $\widehat{\theta}_{\mathrm{NoisyGD}}$ be a predictor trained by NoisyGD under the cross entropy loss with zero initialization.
   Then, for $\beta$ such that $1 - \beta^2 p >0$, we have the following error bound on the misclassification error.
\begin{align*}
    \Pr[y\widehat{\theta}_{\mathrm{NoisyGD}}^T x < 0]  \leq \exp\left(-\frac{n^2(1 - \beta^2p)^2}{(1 + \beta^2p)\sigma^2} \right).
\end{align*}
As a result, to make the misclassification error less than $\gamma$, the sample complexity for $n$ is $O\left( \frac{(1 + \beta^2p)\log\frac{1}{\gamma}}{2\rho (1 - \beta^2p)^2} \right).$ 
\end{theorem}

\paragraph{Remark.}
This misclassification error also corresponds to $\beta^2p$, which is similar to the stochastic case.
When $\beta^2p < 1$, the misclassification error decays exponentially and the misclassification error is dimension-independent.

\paragraph{Multiple iterations.} For the multi-iteration case, we consider the projected NoisyGD to bound the parameters. Precisely, the output is defined iteratively as
\begin{align}
\label{eq:ProjectedNoisyGD}
    \theta_{k+1} = \mathcal{P}_{B_2^p(0,R)}\left(\theta_k - \eta\left(g_n(\theta_k) + \xi_k \right)  \right),
\end{align}
where $B_2^p(0,R)\subset \mathbb{R}^p$ is an $\ell_2$-norm ball with radius $R,$ $\xi_k \sim \mathcal{N}(0,\sigma^2 I_p),$ and $\mathcal{P}_A$ is the projection onto a convex set $A$ w.r.t.\ the Euclidean inner product.
Here we take $\sigma^2 = (1 + \beta^2p)/2\rho$ and the overall privacy budget is $k\rho.$

\begin{theorem}[Multiple iterations]
\label{thm:multi-projected-GD}
    Let $\theta_{k+1}$ be the output of projected NoisyGD defined in \eqref{eq:ProjectedNoisyGD}. For any $t>0$, if we take $\eta = \frac{R}{n(1 + \beta^2p) + (p + \sqrt{pt} +t)},$ then, the misclassification error is 
    \begin{align*}
        &\Pr\left[\theta_{k+1}^T(e_1 + v)<0\right] \leq \exp\left(-\frac{n^2}{C_{p,k}^2 \sigma^2(1 + \beta^2p)} \right) + k e^{-t},
    \end{align*}
    where $C_{p,k} = \frac{1 + 2^{-k}}{1 - 2^{-k}} \cdot \frac{1 - 1/2}{1 + 1/2} \cdot \frac{\left(1 + e^{R(1 + \beta^2p)}\right)^2}{(1 - \beta^2 p)^2}$ and $\sigma^2 = (1 +\beta^2p)/2\rho.$
    Specifically, for $t = n^2+\log 1/k,$ we have
    $
        \Pr\left[\theta_{k+1}^T(e_1 + v)<0\right] \leq O\left( e^{-\frac{\rho n^2}{(1 + \beta^2p)^2}}\right).
$
\end{theorem}

\paragraph{Remark.} To make the projected NoisyGD converge exponentially, we still require $\beta^2 < \frac{1}{p}.$

\subsection{Robustness of NoisyGD under perturbations}
\label{sec:robustness}

In this section, we explore the robustness of NoisyGD against various perturbation types.
For the sake of brevity, we focus on perturbations of the perfect feature ($\beta=0$) by different attackers. The theoretical framework can easily be extended to include perturbations of actual features, following the same proof structure as outlined in Theorem \ref{thm:GD-random perturbation} and Theorem \ref{thm:NoisyGD-random perturbation}.
Our findings indicate that each mentioned perturbation type affects the feature shift parameter $\beta$, potentially increasing NoisyGD's dimension dependency. 

\subsubsection{Stochastic attackers}
\paragraph{Non-robustness to perturbations in the training time.} If the training feature is perturbed by some stochastic perturbation (while the testing feature is perfect), then, the misclassification error for GD is $\exp\left(-\frac{n^2}{\widetilde{\beta}^2}\right)$, which is dimension-independent for any $\widetilde{\beta}>0.$ 
However, the sample complexity for NoisyGD  is $O\left(\sqrt{\frac{\max\{\widetilde{\beta}^2p, 1\}\log(1/\gamma)}{ \rho}}\right).$
Thus, the NoisyGD is non-robust even when we only perturb the training feature with attackers that make $\widetilde{\beta} > p^{-1/2}.$
We postpone the details to Appendix \ref{sec:ran-perturb-train}

\paragraph{Non-robustness to perturbations in the testing time.} If we only perturb the testing feature, then we still require $
O\left(\frac{\max\{\sqrt{p} \widetilde{\beta}, 1\} \sqrt{\log(1/\gamma)}}{\sqrt{2\rho}}\right) 
$ samples to achieve a misclassification error $\gamma$, which is still non-robust when $\widetilde{\beta} > p^{-1/2}.$
The  technical detail is similar to the proof of Theorem \ref{thm:NoisyGD-random perturbation}.

\subsubsection{Deterministic attackers}


\paragraph{Non robustness to offset perturbations in the training time.}
Even if we just shift the training feature vectors away by a constant offset (while keeping the same margin and angle between features), it makes DP learning a lot harder. 
Precisely, for some vector $v\in\mathbb{R}^p$, we consider $v_i = v$ for $y_i = 1$ and $v_i = -v$ for $y_i = -1$.
Moreover, this makes absolutely no difference to the gradient, when we start from $0$ because
$$
\nabla \cL(\theta)  =  \frac{n}{2}\cdot 0.5\cdot -(-e_1 + v)  + \frac{n}{2} \cdot 0.5 \cdot (e_1 + v)  = \frac{n}{2} e_1.
$$
Thus, for GD, the misclassification error is always $0$.
If all we know is that $\|v\|_\infty \leq \widetilde{\beta}$, the sample complexity for making the classification error less than $\gamma$ will be $O\left(\frac{\max\{p\widetilde{\beta}^2,1\}\sqrt{\log(1/\gamma)}}{\sqrt{\rho}}\right).$
The details are given in Appendix \ref{sec:perturb-training-fixed}.

\paragraph{Non-robustness to class imbalance.}
Note that in the above case, it is quite a coincidence that $v$ gets cancelled out in the non-private gradient.
When the class is not balanced, the offset $v$  will be part of the gradient that overwhelms the signal. Consider the case where we have $\alpha n$ data points with label $-1$ and $(1-\alpha) n$ data points with label $1$ for $\alpha \neq 0.5$, and we start at $0$, then 
$
\nabla \cL(\theta)  = \frac{n}{2} e_1 +  \frac{(1-2\alpha)n}{2} v.
$
If we allow the perturbation $v $ to be adversarially chosen, then there exists $v$ satisfying $\|v\|_\infty\leq \widetilde{\beta}$ such that the sample complexity bound to achieve a misclassification error $\gamma$  is $O\left(\frac{  \max\{\sqrt{p}\widetilde{\beta}, 1\}\sqrt{ \log\frac{1}{\gamma}}}{\sqrt{(1 - \widetilde{\beta} + 2\widetilde{\beta}\alpha)^2\cdot \rho}} \right).$

\paragraph{Non-robustness to adversarial perturbations in the testing time.}  When $v_i = 0$ and  $\|v\|_{\infty} \leq \beta,$
that is, we only consider perturbations in the testing time, if we allow the perturbation $v$ to be adversarially chosen, then there exists $v$ satisfying $\|v\|_\infty\leq \widetilde{\beta}$ such that the sample complexity bound to achieve misclassification rate $\gamma$ is $O\left(\frac{G  \max\{p\widetilde{\beta}, 1\} \sqrt{\log(1/\gamma)}}{\sqrt{2\rho}}\right).$
One may refer to Appendix \ref{sec:adv-perturb-test} for the detail.

\section{Solutions for non-robustness issues}
\label{sec:remedy}

In this section, we will explore various solutions for enhancing the robustness of NoisyGD.
To deal with random perturbations, we suggest performing dimension reduction to reduce the feature shift parameter $\beta$, as detailed in Section \ref{sec: pca}.
For offset perturbations, we will consider feature normalization to cancel out the perturbation, as discussed in Section \ref{sec:normalization}.
The proof of this section can be found in Appendix \ref{sec:proof-remedy}.

\subsection{Mitigating random perturbation: dimension reduction}\label{sec: pca}
In \citet{abadi2016DPSGD}, dimension reduction methods, such as DP-PCA, were employed to enhance the performance of deep models. In this section, we demonstrate that applying PCA to the private features effectively improves robustness against random perturbations.
Since we have a public dataset for pre-training a model, we consider performing dimension reduction with this public dataset.
Similar to the PCA method, \cite{DBLP:conf/satml/PintoHYS24,DBLP:conf/alt/NguyenUZ20} demonstrate that DP learning also achieves dimension-free learning bounds on 0-1 losses by applying random projections or a transformation matrix to the features.

To perform dimension reduction, our goal is to generate a projection matrix $\widehat{P} = [\widehat{P}_1, \ldots, \widehat{P}_{K-1}] \in \mathbb{R}^{p \times (K-1)}$ and train with the dataset ${(\widetilde{x}_i = \widehat{P}x_i, y_i)}_{i=1}^n$. It is obvious that the ``best projection" is one where the space spanned by $\widehat{P}$ matches the space spanned by $\{M_i\}_{i=1}^{K}$, with $M_i$ being the perfect feature of the $i$-th class (the $i$-th column of an ETF).

In practice, it is not possible to obtain $\{M_i\}_{i=1}^{K}$ directly, and $\widehat{P}$ needs to be generated using another dataset $\{(\widehat{x}_i, \widehat{y}_i)\}_{i=1}^m$.

Recalling the binary classification problem with a training set $\{(x_i, y_i)\}_{i=1}^n$ and $v_i$ as the feature shift vector of $x_i$, as discussed in Section \ref{sec:robustness}, even in the case of class balance, the accuracy is not robust when $\beta^2 \geq 1/p$. Consider a projection vector $\widehat{P} = e_1 + \Delta$ with $\Delta$ satisfying $\|\Delta\|_{\infty} \leq \beta_0$. The following theorem shows that for $\beta\beta_0 < \frac{1}{p}$, the misclassification error decays exponentially and is dimension-independent.

\begin{theorem}
\label{thm:pert}
    For the NoisyGD trained with $\{(\widetilde{x}_i,y_i)\}_{i=1}^n$, the sample complexity to achieve a misclassification error $\gamma$ is $
        n = O\left(\sqrt{\frac{G_{\beta,\beta_0,p}^2 \log\frac{2}{\gamma}} {M_{\beta,\beta_0,p}\rho}} \right)$
with $G_{\beta,\beta_0,p} = 1 + \beta(1 + \beta_0+p\beta_0)$ and $M_{\beta,\beta_0,p} = (1 - \beta_0)^2 - p\beta\beta_0 - (1 + \beta_0)(\beta+ \beta_0 p) - (\beta + \beta_0p)(1 + \beta + \beta_0 + \beta \beta_0 p)$.
\end{theorem}

\paragraph{Remark.} Theorem \ref{thm:pert} suggests that dimension reduction can relax the requirement from $\beta^2p \leq 1$ to $\beta\beta_0p\leq 1.$ Thus, dimension reduction can enhance robustness whenever $\beta_0 < \beta$. Typically, $\beta_0$ is relatively small and tends to $0$ as $m$ increases.

\label{sec:projection}
The next question is how to construct the projection matrix $\widehat{P} = [\widehat{P}_1,\cdots, \widehat{P}_{K-1}] \in \mathbb{R}^{p\times (K-1)}$.
We introduce the follwoing two methods.

\paragraph{Principle component analysis.}
Let $\{\widehat{P}_j\}_{j=1}^{K-1}$ be the the eigenvectors corresponding to $K-1$ largest eigenvalues of $\widehat{\Sigma}=\frac{1}{m}\sum_{i=1}^m \widehat{x}_i \widehat{x}_i^T.$
For the binary case ($K=2$), we have $\widehat{\Sigma}$ converges to $\Sigma = e_1e_1^T + \widehat{\beta}^2 I_p$ for some constant $\widehat{\beta}.$ Note that the eigenvector corresponding to the largest eigenvalue of $\Sigma$ is the perfect feature $e_1.$
As $\beta_0$ is the infinity norm of $\Delta$, we use a bound on the infinity norm of eigenvectors \citep{DBLP:journals/jmlr/FanWZ17}.
We state the results for $K=2$ that can be extended to $K>2.$
Precisely, for $K=2$, let $\widehat{P}$ be the eigenvector of $\frac{1}{m}\sum_{i=1}^m \widehat{x}_i\widehat{x}_i^T$ that corresponds to the largest eigenvalue.
Then, it holds
$
        \beta_0 = \| \widehat{P} - e_1\|_\infty \leq O\left(\frac{1}{\sqrt{m}}\right)
$
    with probability $O\left(p e^{-m^2}\right).$

\paragraph{Releasing the mean of features.}
Let $X_{k} = \{\widehat{x}_i: \widehat{y}_i \hbox{ belongs to the $k$-th class}\}.$
Let $\widehat{P}_k = \frac{1}{m_k}\sum_{\widehat{x}_i\in X_k} \widehat{x}_i$ with $m_k$ being the size of $X_k$.
Then, we have $\Delta = \widehat{P}_k - M_k = \frac{1}{m_k}\sum_{\widehat{x}_i \in X_k} \widehat{x}_i.$
By the concentration inequality, we have $\beta_0 = \|\Delta\|_\infty \leq O\left(\frac{\widetilde{\beta}}{\sqrt{m_k}}\right)$ with probability $pe^{-{m_k^2}}.$

\subsection{Addressing offset perturbations: normalization}
\label{sec:normalization}
In this section, we explore how feature normalization influences NoisyGD, examining it through the lens of neural collapse. Recent studies \citep{DBLP:conf/iclr/SunSM24} have also demonstrated that data normalization can improve the performance of differential privacy (DP) training.

Recall the shift perturbation where, for $x_i\in X_k$, we have $x_i = \widetilde{M}_k:= M_k+ v$ with some fixed vector $v$.

To deal with the offset perturbation $v$, we pre-process the feature as $\widetilde{x}_i = x_i - \frac{1}{n}\sum_{j=1}^n x_j$.
Then, if the class is balanced, it holds
$
    \widetilde{x}_i = \widetilde{M}_k - \frac{1}{K}\sum_{j=1}^K \widetilde{M}_j = M_k$  for $x_i \in X_k.$
That is, the perturbations canceled out.

We still need to bound the sensitivity of the gradient when training with $\{(\widetilde{x}_i, y_i)\}_{i=1}^n.$
If we delete arbitrary $(x_j,y_j)$ from the dataset, then for the case $K=2$ with data balance, the sensitivity of the gradient is $
    G = \frac{n}{n-1}\leq 2,
$
which is upper bounded by a dimension-independent constant. The sample complexity to achieve the misclassification error $\gamma$ is $O\left(\frac{\sqrt{\log(1/\gamma)}}{\sqrt{\rho}}\right),$ which is dimension independent.

Note that this normalization method is not robust to class imbalance.
In fact, if we consider the class imbalanced case with which we have $\alpha n$ data points with label $+1$ and the rest $(1 - \alpha)n$ data points with label $-1$ for some $\alpha>0$, then we have $\widetilde{x}_i = 2(1 - \alpha) e_1$ for $y_i = 1$ and $\widetilde{x}_i = -2 \alpha e_1$ for $y_i = -1.$
In this class-imbalance case, one can recover the feature embedding $e_1$ and $-e_1$ by considering $\frac{\widetilde{x}_i}{\|\widetilde{x}_i\|_2}.$ 
However, in this case, the sensitivity remains a constant $G_\alpha$ which, although independent of the dimension $p$, still relies on $\alpha$.

\section{Experiments}
\label{sec:experiments}

In this section, we will conduct four sets of experiments to validate our theoretical analysis in Section \ref{sec:theory} and Section \ref{sec:remedy}.
We first discuss the details of Figure \ref{fig:nc_cifar10}, which empirically evaluates the neural collapse level using the feature shift parameter with different pre-trained models in Section \ref{sec:empirical-feature-shift}.
The second experiment, outlined in Section \ref{sec:experiment-perfect-feature}, focuses on the synthetic perfect neural collapse scenario.
In the third experiment, detailed in Section \ref{exp: real}, we empirically investigate the behavior of NoisySGD (with fine-tuning the last layer) under different robustness settings and demonstrate that NoisySGD is ``almost" dimension-independent, but this ``dimension-independency" is not robust to minor perturbations in the last-layer feature.
In the last experiment, discussed in Section \ref{sec:expe-pca}, we demonstrate that dimension reduction methods such as PCA effectively enhance the robustness of NoisySGD.

\subsection{Evaluate the feature shift vector}
\label{sec:empirical-feature-shift}

In this section, we empirically investigate the feature shift parameter $\beta$ if the pre-trained transformer is the Vision Transformer (ViT) or ResNet-50 and the fine-tuned dataset is CIFAR-10. The results are displayed in Figure \ref{fig:nc_cifar10} in the Introduction section.

Recall the feature shift parameter $\beta = \|x - M_k\|_{\infty}$ that is the $\ell_{\infty}$ distance between the feature $x$ and the perfect feature feature mean $M_k$ of the class $k$. The perfect feature $M_k$ is unknown, thus, we cannot compute the feature shift parameter $\beta = \|x - M_k\|_{\infty}$ exactly. However, we can approximate $M_k$ by the empirical feature mean $\widehat{M}_k$, which is the average of all features in the $k$-th class. We use the following steps to evaluate $\beta$ numerically.
\begin{itemize}
\item If $\cos(\widehat{M}_i, \widehat{M}_j) \approx \frac{1}{K-1}$ for all $1 \leq i,j \leq K$, then we may claim that NC happens and $\widehat{M}_k$ can be regarded as the perfect feature (as implied by the maximal-equiangularity in \eqref{eq:defn-etf}).
\item Suppose we have $n$ data points with features $\{x_i\}_{i=1}^n$. For $x_i$ in the $k$-th class, we empirically calculate $\|x_i - \widehat{M}_k\|_{\infty}$ and plot their distribution. 
\end{itemize}

To instantiate various $\beta$ distributions, we consider two ImageNet pre-trained models (ResNet-50 and the vision transformer (ViT) model). We apply two models to extract training features, applying standard feature normalizations ($||x||_2=1$) and evaluate their empirical $\beta$ across all training samples.

Specifically, we first calculate the training feature mean per class $\widehat{M}_i$ for $i\in [K]$ and evaluate the cosine matrix of $\cos(\widehat{M}_i, \widehat{M}_j)$ for $1\leq i, j \leq K$. We found that all entries of the cosine matrix are close to $1/(K-1)$ (roughly all entries between $[-0.2, 0.2]$ and the cosine median is $0.082$ for the ViT model and $0.112$ for the ResNet-50 model). Therefore, we next evaluate $||x -\widehat{M}_k||_{\infty}$ for all training samples.

\subsection{Fine-tune NoisyGD with synthetic neural collapse feature}
\label{sec:experiment-perfect-feature}

We first generate a synthetic data matrix $X\in \cR^{n\times d}$ with feature dimension $d$ under perfect neural collapse.
The number of classes $K$ is $10$ and the sample size is $n = 10^{4}$.
In the default setting, we assume each class draws $n/K$ data from a column of $K$-ETF,  the training starts from a zero weight $\theta$ and the testing data are drawn from the same distribution as $X$. 
The Gaussian noise is selected such that the NoisyGD is $(1,10^{-4})$-DP.

In Figure~\ref{fig: syn}, we observe that an imbalanced class alone does not affect the utility. However, NoisyGD becomes non-robust to class imbalance when combined with a private feature offset with $||\nu||_\infty =0.1$. Additionally, it is non-robust to perturbed test data with $||\nu||_\infty=0.1$.

\subsection{Fine-tune NoisySGD with real datasets}\label{exp: real}

In this section, we empirically investigate the non-robustness of neural collapse using real datatsets.

Precisely, we fine-tune NoisySGD with the ImageNet pre-trained vision transformer~\citep{vit} on CIFAR-10 for $10$ epochs.
Test features in the perturb setting are subjected to Gaussian noise with a variance of 0.1. The vision transformer produces a 768-dimensional feature for each image. To simulate different feature dimensions, we randomly sample a subset of coordinates or make copies of the entire feature space. 

In Figure~\ref{fig: cifar10}, we observe that while perturbing the testing features degrades the utility of both Linear SGD and NoisySGD, Linear SGD is generally unaffected by the increasing dimension. On the other hand, the accuracy of NoisySGD deteriorates significantly as the dimension increases.

\begin{figure}[ht]
  \centering
  \subfigure[\textbf{Synthetic perfect NC feature $X\in \cR^{n\times d}$} with $K=10, n=10^4$ and $(1.0, 10^{-4})$-DP. In the default setting, we assume each class draws $|n/K|$ data from a column of $K$-ETF,  the training starts from a zero weight $w$ and the testing data are drawn from the same distribution as $X$. \label{fig: syn}]{\includegraphics[width=0.45\textwidth]{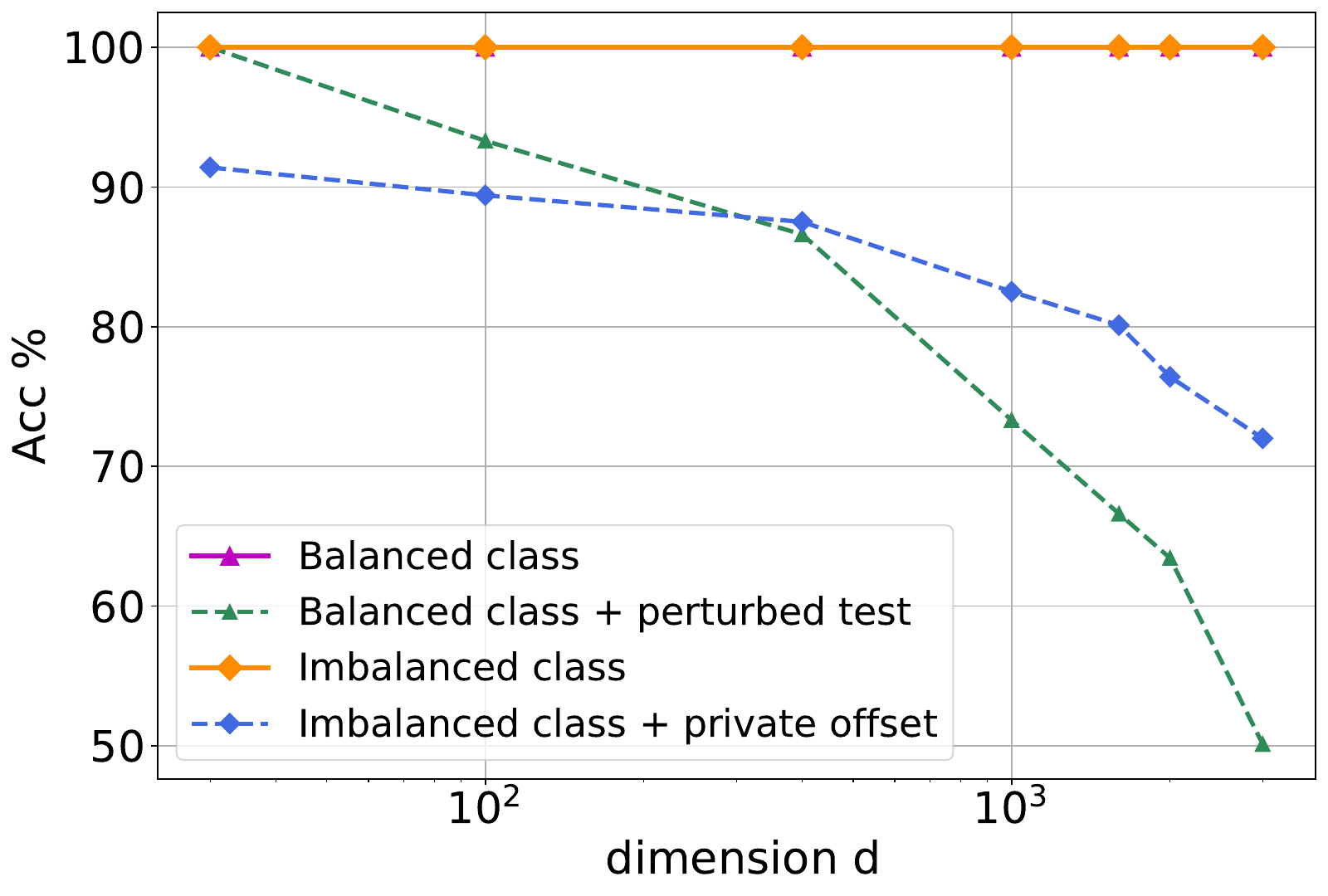} 
  }
  \hspace{1cm}
  \subfigure[\textbf{CIFAR-10}: Test features in the perturb setting are subjected to Gaussian noise with a variance of 0.1. The vision transformer produces a 768-dimensional feature for each image. To simulate different feature dimensions, we randomly sample a subset of coordinates or make copies of the entire feature space. \label{fig: cifar10}]{\includegraphics[width=0.45\textwidth]{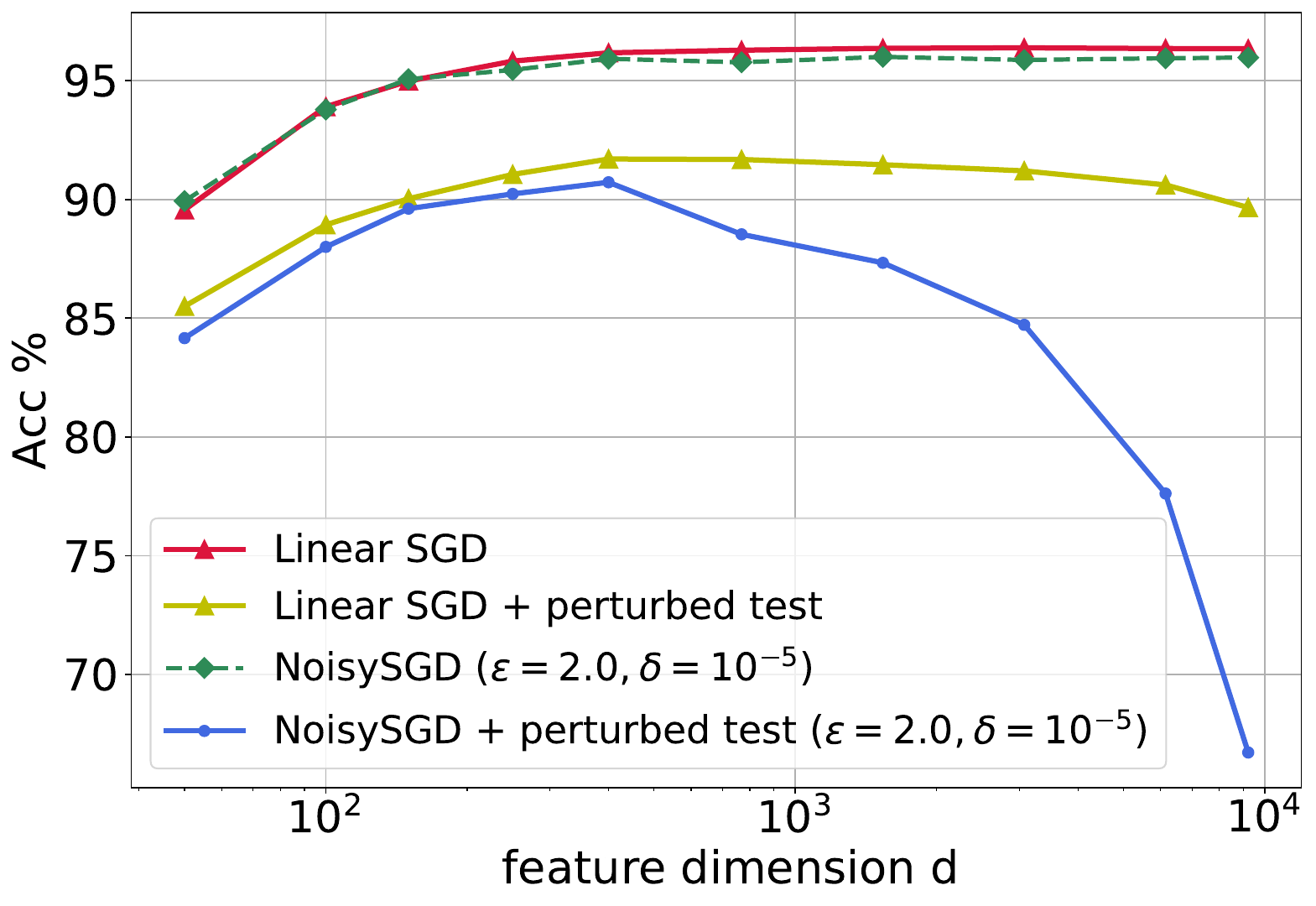}}
  \vspace{-1em}
  \caption{Empirical behaviors of NoisyGD under various robustness setting.}
  \label{fig:example}
  \vspace{-1em}
\end{figure}

\subsection{Enhance NoisySGD's robustness with PCA}
\label{sec:expe-pca}

In this experiment, we replicate the set up from Exp~\ref{exp: real},  simulating different feature dimensions and injecting Gaussian noise with a variance of $0.1$ to perturb all dimensions of both training and testing features. For simplicity, we apply PCA to the covariance matrix of private feature instead of a DP-PCA, followed by principal component projections on both private and testing features prior to feed them into the neural network. As discussed in Section~\ref{sec: pca}, choosing $K-1$ largest eigenvalues is sufficient to improve the robustness of NoisyGD. Therefore, we consider projecting features onto the top $k\in\{10, 50, 100\}$ components. Figure~\ref{fig: pca} shows that the best utility of NoisyGD is achieved when $k=10$, aligning with our theoretical findings. Moreover,  a larger $k=100$ fails to improve robustness, likely because the additional $90$ principal vectors contribute minimal information and introduce further randomness to the training.
\begin{figure}[ht]
  \centering
  \hspace{1cm}
  \includegraphics[width=0.45\textwidth]{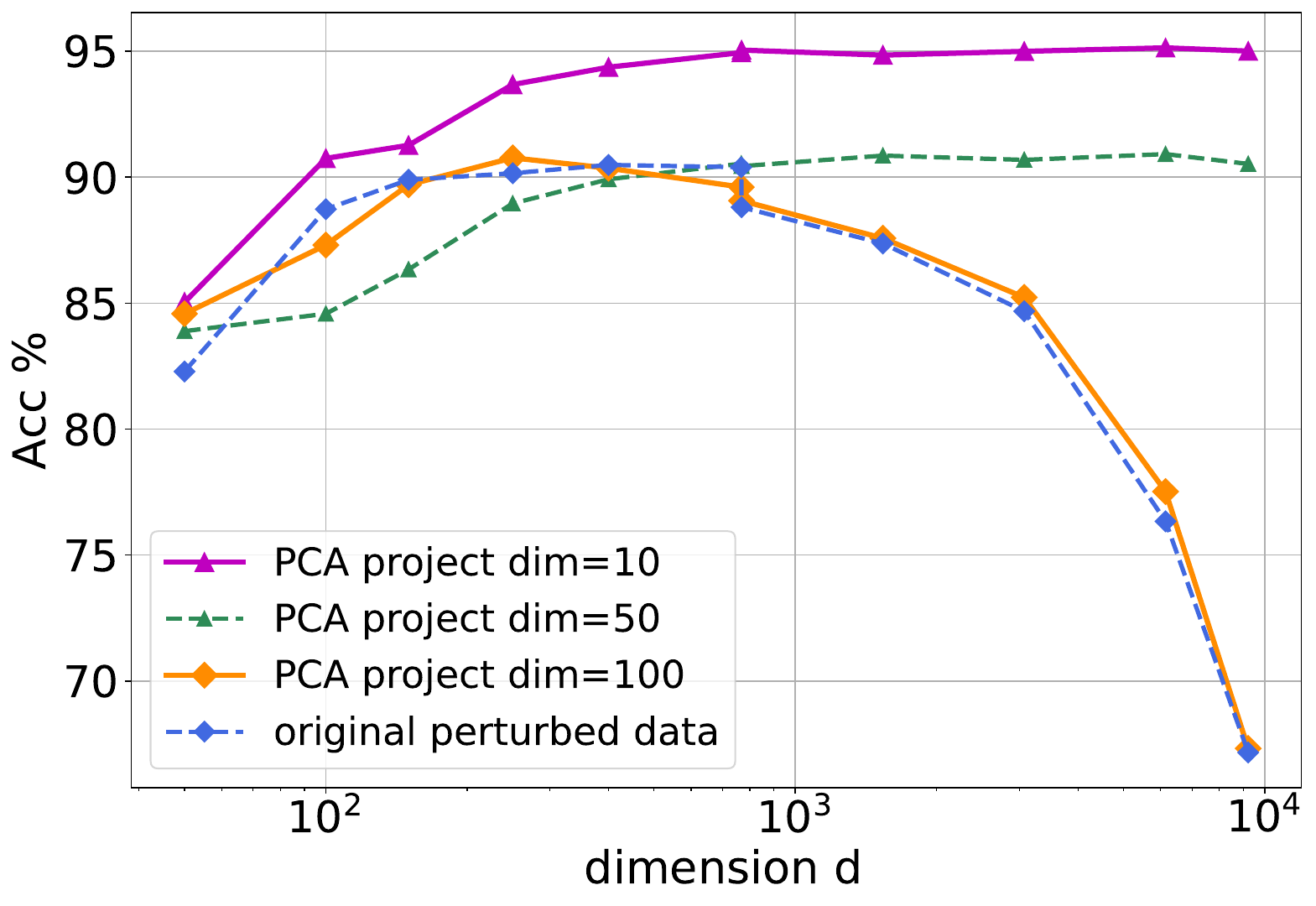}
  \caption{\textbf{CIFAR-10.} Apply PCA on both training and testing features before NoisySGD: setting $K-1$ principal components improves NoisySGD's robustness.}
  \label{fig: pca}
  \vspace{-1em}
\end{figure}

\section{Discussion and future work}
Most existing theory of DP-learning focuses on suboptimality in surrogate loss of testing data. Our paper studies 0-1 loss directly and observed very different behaviors under perfect and near-perfect neural collapse. In particular, we have $\log(1/\hbox{error})$ sample complexity rather than $1/\hbox{error}$ sample complexity.  
Our theoretical findings shed on light that privacy theorists should look into structures of data and how one can adapt to them.
Additionally, our result suggests a number of practical mitigations to make DP-learning more robust in nearly neural collapse settings.  It will be useful to investigate whether the same tricks are useful for private learning in general even without neural collapse.
Moreover, our results suggest that under neural collapse, choice of loss functions (square loss vs CE loss) do not matter very much for private learning. Square loss has the advantage of having a fixed Hessian independent to the parameter, thus making it easier to adapt to strong convexity parameters like in AdaSSP \citep{DBLP:conf/uai/Wang18}.  This is worth exploring.

NoisyGD and NoisySGD theory suggests that one needs $\Omega(n^2)$ time complexity to achieve optimal privacy-utility-trade off in DP-ERM (faster algorithms exist but more complex and they handle only some of the settings).  Our results on the other hand, suggest that when there are structures in the data, e.g., near-perfect neural collapse,  the choice of number of iterations is no longer important, thus making computation easier.

Another perspective in our future study is to investigate when NC will occur in transfer learning. The presence of low-rank structures such as NC in representation learning depends on the pre-trained dataset and the downstream task. Based on our experiments, if the downstream task is CIFAR-10, then ViT performs better than ResNet-50. According to other references in our paper, for example, \cite{masarczyk2023tunnel} show that the low-rank structure is observed if we pre-train on CIFAR-10 and fine-tune on any 10 classes of CIFAR-100. As NC may not consistently occur, we speculate that the collapse level is more significant when the classes of the downstream task closely resemble those of the pre-trained dataset. For instance, NC may manifest when a model is pre-trained on a broad category such as all animals, and the downstream task involves classifying more specific sub-classes (e.g., different breeds of dogs). This intuition needs further empirical investigation with ample computational resources.

    There is another important future topic to consider: what if NC does not occur? When NC cannot be observed, we conjecture that fine-tuning all layers under DP may lead to some low-rank structure of the last layer (potentially inducing some minor collapse phenomenon due to the random noise introduced by DP). Whether NC can be observed after fine-tuning all layers under DP will be validated in our future study.

It is obvious that our theory for NoisyGD can be further generalized using Gaussian differential privacy \citep{MR4400389}, which leads to better utility analysis of our theory. However, extending our theory to NoisySGD is not straightforward, as the privacy accounting becomes complicated when considering sub-sampling due to mini-batches \citep{DBLP:conf/icml/ZhuW19, DBLP:journals/jpc/WangBK20, DBLP:conf/nips/BalleBG18}. If we further consider multiple iterations, the joint effects of sub-sampling and composition on the privacy budget is another challenge \citep{DBLP:conf/aistats/0005DW22}. Besides using composition and the central limit theorem for Gaussian DP \citep{Bu2020Deep,wang2022analytical}, incorporating recently developed privacy analyses of last-iteration output of Noisy(S)GD in terms of the training dynamic \citep{DBLP:conf/nips/0001S22,DBLP:conf/nips/AltschulerT22,bok2024shifted} to obtain refined privacy analyses is another potential future topic. Existing privacy analyses of the last-iteration output are applied to strongly convex loss functions or convex functions with bounded domain, which is applicable to our setting when fine-tuning the last layer.
If we further consider the privacy of NoisySGD under random initialization, the privacy analysis becomes much more complicated, as studied by \citep{ye2023initialization, wang2023unified}.
Moreover, in Appendix \ref{sec:ran-init}, we discussed some extensions of our utility theory to the random initialization case, which shows that the utility theory is much more sophisticated than in the $0$-initialization case.


In addition to DP-ERM, there has been notable recent interest in using public data to enhance the accuracy of DP synthetic data \citep{ghalebikesabi2023differentially,DBLP:conf/icml/LiuV0UW21}. The incorporation of public data, either for traditional query-based synthetic data methods \citep{DBLP:journals/jpc/McKennaMS21,li2023statistical} or more recent techniques such as DP generative adversarial networks or diffusion models \citep{ 10.1145/3422622,DBLP:conf/nips/HoJA20}, has shown promise. The possibility of extending our perspective from Neural Collapse on DP-ERM to DP synthetic data in future research is intriguing.

\subsection*{Acknowledgments}
The research was partially supported by NSF Awards No. 2048091 and No. 2134214.

{\small

\bibliographystyle{apalike}
\bibliography{refs.bib}
}

\newpage
\appendix
\begin{center}
{\Large \bf Appendix}
\end{center}

\section{Proofs of Section \ref{sec:fixedperturb}}
In this section, we provide the proofs of Section \ref{sec:fixedperturb}.
Without loss of generality, we let $v_i^Te_1 = 0$.
Otherwise, by the orthogonal decomposition, there is a constant $c>0$ and a shift vector $v_i$ such that $y_ix_i = c(e_1 + v_i)$ and it is obvious that a scalar $c$ will not change the misclassification error based on our proof details.

\label{sec:proof-perturb}

\subsection{Proof of Theorem \ref{thm:GD-random perturbation}}
\label{sec:proof-GD-random}
The output of linear GD without DP guarantee is given by
\begin{align*}
    \widehat{\theta} = \eta \sum_{i=1}^n y_i x_i.
\end{align*}
For a testing data point $(x,y)$, without loss of generality, we consider $y = 1$ and $x = e_1 + v$ for some vector $v$.
When $v$ is fixed, then, the misclassification error is 0 if $1 - \beta^2p >0.$

For $v_i$ being symmetric i.i.d.\ random vectors, since $v_i^T e_1 = 0$ and $v^Te_1 = 0$, we have the misclassification error is given by
\begin{align*}
    \Pr_{v,v_i}\left[-\left(ne_1 + \sum_{i=1}^n v_i\right)^T(e_1 + v) >0 \right] \leq  \Pr\left[-\sum_{i=1}^n v_i^Tv> n \right].
\end{align*}

Since $|v_i^Tv|\leq \beta^2p$ and $\mathbb{E}[(v_i^Tv)^2]\leq \beta^4p^2$, by a Bernstein-type inequality, we have
\begin{align*}
    \Pr\left[-\sum_{i=1}^n v_i^Tv> n \right] &= \mathbb{E}_v\Pr\left[\left. -\sum_{i=1}^n v_i^Tv> n\right|v\right] \leq \exp\left(-\frac{n^2}{2\left(\beta^4np^2 + \frac{1}{3}\beta^2pn\right)}\right).
\end{align*}

Rewrite $v_i = (v_i^j)_{j=1}^p$ and $v = (v^j)_{j=1}^p$ for $v_i^j, v^j\in [-\beta,\beta].$ 
Note that $|v_i^jv^j|\leq \beta^2$ and $\mathbb{E}[(v_i^jv^j)^2|v^j] \leq \beta^4.$
If we further assume that and $v_i^j, 1\leq i \leq n, 1\leq j \leq p$ are independent random variables, then
, by a Bernstein-type inequality, we have
\begin{align*}
    \Pr\left[-\sum_{i=1}^n v_i^Tv> n \right] &= \mathbb{E}_v\Pr\left[\left. -\sum_{i=1}^n v_i^Tv> n\right|v\right] = \mathbb{E}_v\Pr\left[\left. -\sum_{i=1}^n \sum_{j=1}^p v_i^jv^j> n\right|v\right]
    \\
    &\leq \exp\left(-\frac{n^2}{2\left(\beta^4np + \frac{1}{3}\beta^2n\right)}\right).
\end{align*}


\subsection{Proof of Theorem \ref{thm:NoisyGD-random perturbation}}
\label{sec:proof-fixed-perturbation}

The output of NoisyGD is given by
\begin{align*}
    \widehat{\theta}_{\mathrm{NoisyGD}} = \eta \left(\sum_{i=1}^n y_i x_i + \mathcal{N}(0,\sigma^2 I)\right).
\end{align*}
For a testing data $x = e_1 + v$, we have
\begin{align*}   y\widehat{\theta}_{\mathrm{NoisyGD}}^T x = \mu_n + \mathcal{N}(0,\eta^2\|x\|_2^2 \sigma^2),
\end{align*}
where $\mu_n = \eta (n + \sum_{i=1}^n v_i^T v).$

By the law of total probability, we have
\begin{align*}
    \Pr[y\widehat{\theta}_{\mathrm{NoisyGD}}^T x < 0] =& \Pr\left[\mu_n +\mathcal{N}(0,\eta^2\|x\|_2^2\sigma^2) <  0\right]
    \\
    =& \Pr\left[\mu_n +\mathcal{N}(0,\eta^2\|x\|_2^2\sigma^2) <  0 \bigg| \sum_{i=1}^n v_i^T v \leq -\frac{n}{2}\right] 
    \cdot\Pr\left[\sum_{i=1}^n v_i^T v \leq -\frac{n}{2}\right] 
    \\
    &+ \Pr\left[\mu_n +\mathcal{N}(0,\|x\|_2^2\sigma^2) <  0 \bigg| \sum_{i=1}^n v_i^T v > -\frac{n}{2}\right] 
    \cdot\Pr\left[\sum_{i=1}^n v_i^T v > -\frac{n}{2}\right].
\end{align*}

For the first term, if $\sum_{i=1}^n v_i^T v > -\frac{n}{2},$ we have $\mu_n \geq \frac{\eta n}{2}.$
As a result, it holds
\begin{align*}
    \Pr\left[\mu_n +\mathcal{N}(0,\|x\|_2^2\sigma^2) <  0 \bigg| \sum_{i=1}^n v_i^T v >- \frac{n}{2}\right] &\leq \Pr\left[\mathcal{N}(0,\|x\|_2^2\sigma^2) <  -\frac{n}{2} \bigg| \sum_{i=1}^n v_i^T v >-\frac{n}{2}\right]
    \\
    &\leq \exp\left(-\frac{n^2}{4\|x\|_2^2\sigma^2} \right) \leq \exp\left(-\frac{n^2}{4(1 + \beta^2p)\sigma^2} \right).
\end{align*}
Thus, we have
\begin{align*}
    \Pr\left[\mu_n +\mathcal{N}(0,\|x\|_2^2\sigma^2) <  0 \bigg| \sum_{i=1}^n v_i^T v >- \frac{n}{2}\right] 
    \cdot\Pr\left[\sum_{i=1}^n v_i^T v >- \frac{n}{2}\right] \leq \exp\left(-\frac{n^2}{4(1 + \beta^2p)\sigma^2} \right).
\end{align*}

For the second term, similarly to the proof in Section \ref{sec:proof-GD-random}, we have
\begin{align*}
    \Pr\left[\sum_{i=1}^n v_i^T v \leq -\frac{n}{2}\right]\leq \exp\left(-\frac{n^2}{2\left(\beta^4np^2 + \frac{1}{3}\beta^2np\right)}\right),
\end{align*}
or, under i.i.d.\ assumptions on the components of $v$, we have
\begin{align*}
    \Pr\left[\sum_{i=1}^n v_i^T v \leq -\frac{n}{2}\right]\leq \exp\left(\frac{-n^2}{2\left(\beta^4np + \frac{1}{3}\beta^2n\right)}\right).
\end{align*}
As a result, it holds
\begin{align*}
    \Pr\left[\mu_n +\mathcal{N}(0,\|x\|_2^2\sigma^2) <  0 \bigg| \sum_{i=1}^n v_i^T v \leq -\frac{n}{2}\right] 
    \cdot\Pr\left[\sum_{i=1}^n v_i^T v \leq - \frac{n}{2}\right]\leq \exp\left(-\frac{n^2}{2\left(\beta^4np^2 + \frac{1}{3}\beta^2np\right)}\right),
\end{align*}
or, under further i.i.d.\ assumptions on the components of $v$, it holds
\begin{align*}
    \Pr\left[\mu_n +\mathcal{N}(0,\|x\|_2^2\sigma^2) <  0 \bigg| \sum_{i=1}^n v_i^T v \leq -\frac{n}{2}\right] 
    \cdot\Pr\left[\sum_{i=1}^n v_i^T v \leq -\frac{n}{2}\right]\leq \exp\left(-\frac{n^2}{2\left(\beta^4np + \frac{1}{3}\beta^2n\right)}\right).
\end{align*}

Over all, we obtain
\begin{align*}
\Pr[y\widehat{\theta}_{\mathrm{NoisyGD}}^T x < 0] \leq \exp\left(-\frac{n^2}{4(1 + \beta^2p)\sigma^2} \right) + \exp\left(-\frac{-n^2}{2\left(\beta^4np^2 + \frac{1}{3}\beta^2np\right)}\right),
\end{align*}
or, under further i.i.d.\ assumptions on all $p$ components of $v$, we obtain
\begin{align*}
    \Pr[y\widehat{\theta}_{\mathrm{NoisyGD}}^T x < 0] \leq \exp\left(-\frac{n^2}{4(1 + \beta^2p)\sigma^2} \right) + \exp\left(-\frac{-n^2}{2\left(\beta^4np + \frac{1}{3}\beta^2n\right)}\right).
\end{align*}
We finish the proof by noting that $\sigma^2 = (1 + \beta^2p)/2\rho.$

Let the misclassification error less than $\gamma$ and we have
\begin{align*}
    n = O\left( \frac{(1 + \beta^2p)^2\log\frac{1}{\gamma}}{2\rho} + 4p\beta^2\log\frac{1}{\gamma} \right),
\end{align*}
or, under further i.i.d.\ assumptions on all $p$ components of v, we have
\begin{align*}
    n = O\left( \frac{(1 + \beta^2p)^2\log\frac{1}{\gamma}}{2\rho} + 4p\beta^4\log\frac{1}{\gamma} \right).
\end{align*}

\subsection{Proof of Theorem \ref{thm:fixed-perturbation}}
The output of NoisyGD is given by
\begin{align*}
    \widehat{\theta}_{\mathrm{NoisyGD}} = \eta \left(\sum_{i=1}^n y_i x_i + \mathcal{N}(0,\sigma^2 I) \right).
\end{align*}
For a testing data $x = e_1 + v$, we have
\begin{align*}   y\widehat{\theta}_{\mathrm{NoisyGD}}^T x = \eta (\mu_n + (e_1 + v)^T\xi),
\end{align*}
where $\mu_n = n + \sum_{i=1}^n v_i^T v $ and $\xi\sim\mathcal{N}(0,\sigma^2 I).$

Since $\beta^2p < 1$, we have
\begin{align*}
    \mu_n \geq n(1 - \beta^2p) >0.
\end{align*}
Thus, we have
\begin{align*}   \Pr\left[y\widehat{\theta}_{\mathrm{NoisyGD}}^T x < 0\right]  &\leq \Pr\left[\mathcal{N}(0,\|x\|_2^2\sigma^2) > n(1 - \beta^2p) \right]
\\
&\leq \exp\left(-\frac{n^2(1 - \beta^2p)^2}{(1 + \beta^2p)\sigma^2} \right).
\end{align*}




\subsection{Proof of Theorem \ref{thm:multi-projected-GD}}
\label{sec:multi-iteration}
Now we consider multiple iteration. Recall the loss function $\ell(\theta, (x,y)) = \log(1 + e^{-y \cdot \theta^T x}).$
The gradient is then given by 
\begin{align*}
    g(\theta,(x,y)) = \frac{e^{-y\cdot \theta^T x}}{1 + e^{-y\cdot \theta^T x}} (-y x).
\end{align*}
Note that $yx = e_1 + v$ for $y=1$ and $yx = e_1 - v$ for $y=-1$.
For $\alpha=1/2$, we have
\begin{align*}
    -g_n(\theta) := \nabla \mathcal{L}(\theta) = \frac{n}{2}\left[ \frac{1}{1 + e^{\theta^T(e_1 + v)}} (e_1 + v) + \frac{1}{1 + e^{\theta^T(e_1 -v)}} (e_1 - v)\right].
\end{align*}
and
\begin{align*}
  0  \leq -g_{n}(\theta)^T e_1 \leq \frac{n(1 - \beta)}{2}
\end{align*}
for any $\theta \in \mathbb{R}^p.$
To make the multi-iterations work, we need lower bound $-g_{n}(\theta)^T e_1$ by a positive constant.
Now we consider the two-iteration GD.
\begin{align*}
    \theta_2 = \theta_1 - \eta(g_n(\theta_1) + \xi_1), \qquad \xi_1 \sim \mathcal{N}(0,\sigma^2 I_p).
\end{align*}
Here $\theta_1 = \frac{n}{2}e_1 + \xi_0.$
Then, it holds
\begin{align*}
    g_n(\theta_1) = \frac{n}{2}\left[\frac{e_1 + v}{1 + e^{n/2 + \xi_{0,1}}} + \frac{e_1 - v}{1 + e^{n/2 - \xi_{0,1}}} \right]
\end{align*}

To lower bound the gradient, we consider the projected iteration defined as
\begin{align*}
    \theta_{k+1} = \mathcal{P}_{B_2^p(0,R)}\left(\theta_k - \eta\left(g_n(\theta_k) + \xi_k \right)  \right),
\end{align*}
where $B_2^p(0,R)\subset \mathbb{R}^p$ is an $\ell_2$-norm ball with radius $R,$ $\xi_k \sim \mathcal{N}(0,\sigma^2),$ and $\mathcal{P}_A$ is the projection onto a convex set $A$ w.r.t.\ the Euclidean inner product.

Then we have
\begin{align*}
    \theta_{k+1}^T (e_1+v) = c_{R,k} \left(\theta_k - \eta\left(g_n(\theta_k) + \xi_k \right) \right)^T (e_1+v),
\end{align*}
where $c_{R,k} = \frac{\min\{R, \|\theta_k - \eta\left(g_n(\theta_k) + \xi_k\right)\|_2\}}{\|\theta_k - \eta\left(g_n(\theta_k) + \xi_k\right)\|_2}.$
Since
\begin{align*}   \eta\left\|\left(g_n(\theta_k) + \xi_k \right)\right\|_2 \leq \eta n(1 + \beta^2p) + \eta\|\xi_k\|_2,
\end{align*}
it is enough to bound $\|\xi_k\|_2.$
Note that $\|\xi_k\|_2^2$ is a $\chi^2(p)$ distribution and one (cf., \citet{MR1805785}) has the tail bound
\begin{align*}
    \Pr\left[\eta^2 \|\xi_k\|_2^2 >\eta^2 (p + \sqrt{pt} + t) \right] \leq e^{-t}
\end{align*}
for any $t>0.$
Thus, by the union bound, we have
\begin{align*}
    C_{R,k} \geq C_{R,p,t,\eta,n} = \frac{R}{R + \eta n(1 + \beta^2 p) + \eta(p + \sqrt{pt} + t)^{1/2}}
\end{align*}
for any $k\geq 0$ with probability $ke^{-t}$.
Since $-g_n(\theta_k)^T(e_1+v) \geq n\frac{1 - \beta^2p}{1 + e^{R(1 + \beta^2p)}}$, it holds
\begin{align*}
    \theta_{k+1}^T (e_1+ v) &\geq C_{R,p,t,\eta,n} \left(\theta_{k}^T(e_1+ v) + n\eta \frac{1 - \beta^2p}{1 + e^{R(1 + \beta^2p)}} - \eta\xi_k^T(e_1 + v)\right)
    \\
    & \geq n\eta \frac{1 -\beta^2p}{1 + e^{R(1 + \beta^2p)}} \sum_{j=1}^kC_{R,p,t,\eta,n}^{k-j+1} + \eta\sum_{j=1}^k C_{R,p,t,\eta,n}^{k-j+1} \xi_j^T(e_1 + v).
\end{align*}
Since $\sum_{j=1}^k C_{R,p,t,\eta,n}^{k-j+1} \xi_j^T(e_1+v) \sim \mathcal{N}(0,\sum_{j=1}^k C_{R,p,t,\eta,n}^{2(k-j+1)}\sigma^2\|e_1+v\|_2^2),$ we have
\begin{align*}    \Pr\left[\theta_{k+1}^T(e_1+ v) <0\right] \leq \Pr\left[ \mathcal{N}\left(n\frac{1 - \beta^2p}{1 + e^{R(1 + \beta^2p)}}, \frac{\sum_{j=1}^kC_{R,p,t,\eta,n}^{2(k-j+1)}}{\left(\sum_{j=1}^k C_{R,p,t,\eta,n}^{k-j+1}\right)^2}
\sigma^2(1 + \beta^2p)\right) <0 \right] + ke^{-t}.
\end{align*}
Since $\sum_{j=1}^kC_{R,p,t,\eta,n}^{2(k-j+1)} = \frac{C_{R,p,t,\eta,n}^2(1 - C_{R,p,t,\eta,n}^{2k})}{1 - C_{R,p,t,\eta,n}^2}$ and $\sum_{j=1}^kC_{R,p,t,\eta,n}^{(k-j+1)} = \frac{C_{R,p,t,\eta,n}(1 - C_{R,p,t,\eta,n}^{k})}{1 - C_{R,p,t,\eta,n}},$ we have 
\begin{align*}
    \frac{\sum_{j=1}^kC_{R,p,t,\eta,n}^{2(k-j+1)}}{\left(\sum_{j=1}^kC_{R,p,t,\eta,n}^{(k-j+1)}\right)^2} = \frac{1 + C_{R,p,t,\eta,n}^k}{1 - C_{R,p,t,\eta,n}^k} \cdot \frac{1 - C_{R,p,t,\eta,n}}{1 + C_{R,p,t,\eta,n}}.
\end{align*}
By taking $\eta = \frac{R}{n(1 + \beta^2p) + (p + \sqrt{pt} +t)},$ we get $C_{R,p,t,\eta,n} = \frac{1}{2}$ and
\begin{align*}
    \frac{1 + C_{R,p,t,\eta,n}^k}{1 - C_{R,p,t,\eta,n}^k} \cdot \frac{1 - C_{R,p,t,\eta,n}}{1 + C_{R,p,t,\eta,n}} = \frac{1 + 2^{-k}}{1 - 2^{-k}} \cdot \frac{1 - 1/2}{1 + 1/2}=:C_k. 
\end{align*}
Thus, it holds
\begin{align*}   \Pr\left[\theta_{k+1}^T(e_1+v) <0\right] \leq \exp\left(-\frac{n^2}{C_{p,k}^2 \sigma^2(1 + \beta^2p)} \right) + k e^{-t}
\end{align*}
with  $C_{p,k}^2 = C_k \left(1 + e^{R(1 + \beta^2p)}\right)^2/(1-\beta^2p)^2.$

\section{Results for perturbing only the training data}
\label{proof:pert-train}

\subsection{Fixed perturbation}
\label{sec:perturb-training-fixed}
Without loss of generality, we assume $0<\alpha<1/2.$
Consider the class imbalanced case with $n_{-1} = \alpha n$ and $n_{+1} = (1 - \alpha)n.$
The gradient for $\theta_0 = 0$ is 
$$
\nabla \cL(\theta_0)  =  \alpha n\cdot 0.5\cdot -(-e_1 + v)  + (1-\alpha)n \cdot 0.5 \cdot (e_1 + v)  = \frac{n}{2} e_1 +  \frac{(1-2\alpha)n}{2} v.
$$

Thus, the output is
\begin{align*}
    \widehat{\theta} = -\eta\left( \frac{n}{2}e_1 + \frac{(1 - 2\alpha)n}{2}v + \mathcal{N}(0,\sigma^2)\right)
\end{align*}
The sensitivity is $G = \sqrt{1 + \|v\|_2^2}$ and $\sigma^2$ is taken to be $G^2/2\rho$ to achieve $\rho$-zCDP.
Moreover, we have 
\begin{align*}
    \widehat{\theta}^Te_1 = -\frac{n}{2} - \frac{(1 - 2\alpha)n}{2}v_1 + \mathcal{N}(0,\sigma^2). 
\end{align*}
Thus, the misclassification error is
\begin{align*}
    \Pr[\widehat{\theta}e_1 >0]=\Phi\left(\frac{n\left[ 1 - \left(1  - 2\alpha \right)v_1\right]}{2\sigma}\right)\leq e^{-\frac{n^2(1 - \beta + 2\alpha\beta)^2\rho}{4G^2}}.
\end{align*}
As a result, the sample complexity to achieve $1 - \gamma$ accuracy is 
\begin{align*}
    n = O\left(\sqrt{\frac{4G^2\log\frac{1}{\delta}}{(1 - \beta + 2\beta\alpha)^2\cdot \rho}}\right)
\end{align*}
The sensitivity $G = \sqrt{1 + \beta^2 p}$ here is dimension-dependent.

\subsection{Random perturbation}
\label{sec:ran-perturb-train}
Now we consider the random perturbation. Denote $\{v_i\}_{i=1}^n\subseteq \mathbb{R}^{p}$ a sequence of i.i.d.\ copies of a random vector $v$. 
Consider the binary classification problem with training set $\{(x_i, y_i)\}_{i=1}^n$.
Here $x_i = e_1 + v_i$ if $y_i = 1$ and $x_i = -e_1 + v_i$ if $y_i = -1$.
Then, the loss function is $\mathcal{L}(\theta) = \frac{1}{n}\sum_{i=1}^n \log\left(1 + e^{-y_i \theta^T x_i}\right).$
The one-step iterate of DP-GD from $0$ outputs
\begin{align*}
    \widehat{\theta} = -\eta\sum_{i=1}^n(-y_i x_i) + \mathcal{N}(0,\sigma^2 I_p)
\end{align*}
with $\sigma^2 = G^2/2\rho$ and $G = \sup_{v_i} \sqrt{1 + \|v_i\|^2}.$
Assume that $v_i$ is symmetric, that is $y_i v_i$ has the same distribution as $-y_iv_i$.
Then, it holds
\begin{align*}
    \sum_{i=1}^n y_ix_i = n e_1 + \sum_{i=1}^n v_i=:\mu_n.
\end{align*}
The misclassification error is now given by
\begin{align*}
    \Pr[\widehat{\theta}^Te_1 <0] = \Pr[\mathcal{N}(\mu_n^Te_1,\sigma^2) <0].
\end{align*}

Assume that $\|v_i\|_{\infty} = \beta <1$. Then, we have $\mu_n^T e_1 \geq n - \beta n$ and the sample complexity is $O\left(\sqrt{\frac{4G^2\log(1/\delta)}{(1 - \beta)^2 \rho}}\right)$
with $G = \sqrt{1 + \beta^2 p}$.
\section{Results for perturbing only the testing data}
\label{sec:perturb-test}
For simplicity, we only consider perturbing the perfect feature.
Our results can be extended naturally to the case of actual features.
\subsection{Fixed (offset) perturbation}
\label{sec:fix-perturb-test}

Recall that the output of DP-GD has the form $\widehat{\theta} = \mathcal{N}(-\frac{\eta n}{2},\sigma^2).$
One has
$$
\hat{\theta}^T (e+ v)  = \frac{n}{2} +  \cN\left(0, \frac{G^2 (p\beta^2+1)}{2\rho} \right).
$$
The sample complexity can be derived similarly as previous sections, which is dimension dependent.
\subsection{Adversarial perturbation}
\label{sec:adv-perturb-test}

Let's say in prediction time, the input data point can be perturbed by a small value in $\ell_\infty$.
If we allow the perturbation to be adversarial chosen, then there exits $v$ satisfying $\|v\|_\infty\leq \beta$ such that 
$$\hat{\theta}^T (x+ v)  = \frac{n}{2} +  \frac{G}{\sqrt{2\rho}} Z_1  - \sum_{i=1}^p |Z_i|\frac{G \beta}{\sqrt{2\rho}},$$
where $Z_1,...,Z_p\sim \cN(0,1)$ i.i.d.  
Note that the additional term scales as $O(p \frac{G \beta}{\sqrt{\rho}})$, which can alter the prediction if $p  \asymp n$ even if $\rho$ is a constant (weak privacy). 

The number of data points needed to achieve $1-\delta$ robust classification under neural collapse is 
$
O\left(\frac{G  \max\{p\beta, 1\} \sqrt{\log(1/\delta)}}{\sqrt{2\rho}}\right).
$

\section{Proofs of Section \ref{sec:perfect-collapse}}
\label{appe:proof-thm-2}

\subsection{Omitted details for perfect features}
\label{sec:perfect-collapse}
In this section, we explore the theoretical scenario of an ideal, perfect neural collapse, where $\beta$ equals to zero. Although achieving a perfect neural collapse is practically elusive, examining this idealized case is beneficial due to the promising properties it exhibits. However, as we have noted in the previous sections, these advantageous properties are highly fragile to any perturbations in the features.
The details of this section is given in Appendix \ref{appe:proof-thm-2}.

Consider the multi-class classification problem with $K$ classes.
Under perfect neural collapse, we have $\Pr[x = M_k| y_k=1] = 1$.
Let $\widehat{\theta}\in\mathbb{R}^{Kp}$ be the 1-step output of the NoisyGD algorithm defined in Equation (\ref{eq:DP-GD}) with 0-initialization and let $\widehat{W}\in\mathbb{R}^{K\times p}$ be the matrix derived from $\widehat{\theta}$. Denote $\widehat{y} = \mathrm{OneHot}(\widehat{W} x)\in\{0,1\}^K$ the predictor after the one-hot encoding, that is $\widehat{y}_i = 1$ if $i = \argmax_j\{(\widehat{W}x)_j\}$, otherwise $\widehat{y}_i = 0.$

\begin{theorem}
\label{thm:error-zero-init}
   Let $\widehat{y}$ be a predictor trained by NoisyGD under the cross entropy loss with zero initialization.
   Assume that the training dataset is balanced, that is, the sample size of each class is $n/K$.
   For classification problems under neural collapse with $K$ classes, if we take $G\geq 1$, then the misclassification error is 
   \begin{align*}
       \Pr[\widehat{y}\neq y] &= (K-1)\Phi\left(-\frac{n}{K\sigma}\left(1 + \frac{K-2}{K(K-1)}\right)\right)
       \\
       &\leq (K-1)e^{-\frac{C_K n^2}{K\sigma^2}}
   \end{align*}
with $\sigma^2=\frac{G^2}{2\rho}$ and $C_K = \left(1 + \frac{K-2}{K(K-1)}\right)^2.$
As a result, to make the misclassification error less than $\gamma$, the sample complexity for $n$ is $\frac{GK\sqrt{\log((K-1)/\gamma)}}{C_K\sqrt{2\rho}}.$ 
\end{theorem}
The theorem offers several insights, we have
\begin{enumerate}
	\item The error bound is exponentially close to $0$ if $\rho\gg G^2/n^2$ --- very strong privacy and very strong utility at the same time.
\item The result is dimension independent --- it doesn't depend on the dimension $p$.
\item Even though here we assume that the training dataset is class-balanced, the result is robust to class imbalance for $K=2$, if we apply a re-parameterization of private data (see, Section \ref{sec:fixedperturb}).
	\item The result is independent of the shape of the loss functions. Logistic loss works, while square losses also works. 
	\item The result does not require careful choice of learning rate. Any learning rate works equally well.
\end{enumerate}

Our theory can be extended to the domain adaptation context, where the model is initially pre-trained on an extensive dataset with $K_0$ classes, and is subsequently fine-tuned for a downstream task with a smaller number of classes $K\leq K_0.$
One may refer to Appendix \ref{proof:error-zero-
init-small}.

\subsection{Proof of Theorem \ref{thm:error-zero-init} and corresponding results}
Recall an ETF defined by
\begin{align*}
    M &= \sqrt{\frac{K}{K-1}} P \left(I_K - \frac{1}{K}\mathbf{1}_K \mathbf{1}_K^T\right)=\sqrt{\frac{K}{K-1}} \left( P - \frac{1}{K}\sum_{k=1}^K P_k\mathbf{1}_K^T\right),
\end{align*}
where $P=[P_1,\cdots, P_K]\in\mathbb{R}^{p\times K}$ is a partial orthogonal matrix with $P^T P = I_K$.
Rewrite $M = \left[M_1,\cdots, M_K \right].$
Let the label $y=(y_1,\cdots,y_K)^T\in\{0,1\}^K$ be represented by the one-hot encoding, that is, $y_k=1$ and $y_j=0$ for $j\neq k$ if $y$ belongs to the $k$-th class.
\begin{definition}[Classification problem under Neural Collapse]
	Let there be $K$ classes. The distribution $\P[x  = M_k|y_k=1]  = 1$ for $k=1,...,K$. 
\end{definition} 


\begin{proof}[Proof of Theorem \ref{thm:error-zero-init}.]
Let $W=[W_1,\cdots,W_K]^T\in\mathbb{R}^{K\times p}$. Consider the output function $f_{W}(x) = W x \in\mathbb{R}^K$. Suppost that $y_k=1$. Then, the cross-entropy loss is defined by
\begin{align*}
    \ell(f_{W}(x), y) = -\log\left(\frac{e^{W_k^T x}}{\sum_{k'=1}^K e^{W_{k'}^T x}}\right).
\end{align*}
The corresponding empirical risk is 
\begin{align*}
    R_n(M,W)=\sum_{k=1}^K-n_k\log\left(\frac{e^{W_k^T M_k}}{\sum_{k'=1}^K e^{W_{k'}^T M_k}}\right).
\end{align*}
Note that
\begin{align*}
    \nabla_{W}\ell(f_{W}(x), y)= \left(\mathrm{SoftMax}(f_W(x)) - y \right) x^T,
\end{align*}
where $\mathrm{SoftMax}:\mathbb{R}^K \rightarrow\mathbb{R}^K$ is the SoftMax function defined by
\begin{align*}
    \mathrm{SoftMax}(z)_i = \frac{e^{z_i}}{\sum_{j=1}^K e^{z_j}}, \qquad \hbox{ for all } z\in\mathbb{R}^{K}.
\end{align*}
We obtain
\begin{align*}
    \nabla_WR_n(M,W) = \sum_{k=1}^K n_k\left(\mathrm{SoftMax}(f_W(M_k)) - y^k \right) M_k^T,
\end{align*}
where $y^k$ is the label of the $k$-th class.
For zero initialization, we have
\begin{align*}
    \mathrm{SoftMax}(f_{\mathbf{0}}(M_k))= \frac{1}{K}\mathbf{1}_K
\end{align*}
and
\begin{align}
\label{eq:grad-cross-entrop}
    \nabla_W(R_n(M,W))\Big|_{W=\mathbf{0}} = \sum_{k=1}^K n_k\left(\frac{1}{K}\mathbf{1}_K - y^k\right) M_k^{T}.
\end{align}
Now we consider one step NoisyGD from 0 with learning rate $\eta=1$:
\begin{align*}
    \widehat{W} = -\sum_{k=1}^K n_k\left(\frac{1}{K}\mathbf{1}_K - y^k\right) M_k^{T} + \Xi,
\end{align*}
where $\Xi\in\mathbb{R}^{K\times p}$ with $\Xi_{ij}$ drawn independently from a normal distribution $\mathcal{N}(0,\sigma^2).$

Consider $x = M_{k}$. It holds
\begin{align*}
    f_{\widehat{W}}(x) = \widehat{W}M_k = -\sum_{k'=1}^K n_{k'}\left(\frac{1}{K}\mathbf{1}_K - y^{k'}\right) M_{k'}^{T} M_{k} + \Xi M_k.
\end{align*}
Since
\begin{align*}
    \Xi M_k \sim \mathcal{N}\left(0, \sigma^2\|M_k\|_2^2 I_K\right) \qquad \hbox{and} \qquad \|M_k\|_2^2=1,
\end{align*}
we have
\begin{align*}
    \widehat{W}M_k \sim \mathcal{N}\left(\boldsymbol{\mu}_{n,K}, \sigma^2 I_K\right),
\end{align*}
where $\boldsymbol{\mu}_{n,K} = -\sum_{k'=1}^K n_{k'}\left(\frac{1}{K}\mathbf{1}_K - y^{k'}\right) M_{k'}^{T} M_{k}$.
Note that
\begin{align*}
    M_{k'}^T M_k = \frac{K}{K-1}\left(\delta_{k,k'} - \frac{1}{K}\right).
\end{align*}
We obtain
\begin{align*}
    \left(\boldsymbol{\mu}_{n,K}\right)_j =
    \left\{
    \begin{array}{ll}
     n/K,    & j = k, \\
     -\frac{n(K-2)}{K^2(K-1)},    & j\neq k,
    \end{array}
    \right.
\end{align*}
for $n_{k'}=n/K$ (balanced data).
By the union bound, the misclassification error is 
\begin{align*}
    (K-1)\Pr\left[\mathcal{N}(n/K,\sigma^2) < \mathcal{N}(-\frac{n(K-2)}{K^2(K-1)},\sigma^2)\right]
    =(K-1)\Phi\left(-\frac{n}{K\sigma}\left(1 + \frac{K-2}{K(K-1)}\right)\right).
\end{align*}
\end{proof}

\paragraph{Proof sketches of the insights.}
Note that in Equation \eqref{eq:grad-cross-entrop}, the gradient is a linear function of the feature  map thanks to the zero-initialization while for least-squares loss, one can derive a similar gradient as  Equation \ref{eq:grad-cross-entrop}.
Thus, the proof can be extended to the least squares loss directly.
Moreover, by replacing $n_k$ with $n_k\eta$ in \eqref{eq:grad-cross-entrop}, one can extend the results to any $\eta.$

\subsection{Omited details for binary classification}
\label{proof:repara}
Recall the re-parameterization for $K=2$. Precisely, an equivalent neural collapse case gives $M =[-e_1, e_1]$ with $e_1 = [1, 0,\dots, 0]^T.$
Furthermore, we consider the re-parameterization with $y\in \{-1,1\}$,  $\theta\in \R^p$ and the decision rule being $\hat{y} = \sign(\theta^T x)$. Then, the logistic loss is $\log(1 + e^{-y \cdot \theta^Tx})$. 

\begin{theorem}
\label{thm:repara}
    For the case $K=2$, under perfect neural collapse, using the aforementioned re-parameterization,
    the misclassification error is 
    \begin{align*}
      \Pr[\widehat{y} \neq y] =  \Phi\left(-\frac{n}{2\sigma}\right) \leq e^{-\frac{n^2}{2\sigma^2}}.
    \end{align*}
    Moreover, to make the misclassification error less than $\gamma,$
     the sample complexity is $O\left(\frac{\sqrt{\log(1/\gamma)}}{\sqrt{\rho}}\right)$.
\end{theorem}

\noindent\paragraph{Remark.} The bound derived here is optimal up to a $\log n$ term. In fact, we have
\begin{align*}
    \Pr[\widehat{y} \neq y] =  \Phi\left(-\frac{n}{2\sigma}\right) \geq \frac{2\sigma}{n}e^{-\frac{n^2}{2\sigma}} = e^{-\frac{n^2}{2\sigma^2} + \log\frac{2\sigma}{n}}.
\end{align*}

\begin{proof}
    According to the re-parameterization, for the class imbalanced case, we have

\begin{align*}
\hat{\theta} =  - \eta \left(\frac{n}{2} \cdot 0.5 \cdot (- \begin{bmatrix}
	-1\\
	0\\
	\vdots\\
	0\\
\end{bmatrix})  + \frac{n}{2} \cdot 0.5 \cdot \begin{bmatrix}
1\\
0\\
\vdots\\
0\\
\end{bmatrix} + \cN(0, \frac{G^2}{2\rho} I_p)\right) = -\eta \left(\begin{bmatrix}
n/2 \\
0\\
\vdots\\
0\\
\end{bmatrix}  + \cN(0, \frac{G^2}{2\rho} I_p)\right).
\end{align*}

The rest of the proof is similar to that of Theorem \ref{thm:error-zero-init}.

For the class-imbalanced case, assume that we have $\alpha n$ data points have with label $+1$ while the rest $(1 - \alpha)n$ points have label $-1$.
Then, the gradient is 

\begin{align*}
\hat{\theta} =  - \eta \left(\frac{n\alpha}{2} \cdot \cdot (- \begin{bmatrix}
	-1\\
	0\\
	\vdots\\
	0\\
\end{bmatrix})  + \frac{n(1 - \alpha)}{2} \cdot \cdot \begin{bmatrix}
1\\
0\\
\vdots\\
0\\
\end{bmatrix} + \cN(0, \frac{G^2}{2\rho} I_p)\right) = -\eta \left(\begin{bmatrix}
n/2 \\
0\\
\vdots\\
0\\
\end{bmatrix}  + \cN(0, \frac{G^2}{2\rho} I_p)\right).
\end{align*}
Thus, the same conclusion holds.
\end{proof}

\subsection{Domain adaption}
\label{proof:error-zero-
init-small}

\noindent\paragraph{Neural collapse in domain adaptation:} In many private fine-tuning scenarios, the model is initially pre-trained on an extensive dataset with thousands of classes (e.g., ImageNet), denoted as $K_0$ class, and is subsequently fine-tuned for a downstream task with a smaller number of classes, denotes as $K \leq K_0$. 
 We formalize it under the neural collapse setting as follows.

Let $P=[P_1,\cdots, P_{K_0}]\in\mathbb{R}^{p\times {K_0}}$ be a partial orthogonal matrix with $P^T P = I_{K_0}$.
Let $\widetilde{M} =[\widetilde{M}_1,\cdots, \widetilde{M}_K]$ be a matrix where each $\widetilde{M}_i$ is a column of an ETF $M\in\mathbb{R}^{p\times K_0}.$
With prefect neural collapse, we assume $\Pr[x = \widetilde{M}_k| y_k = 1] = 1$. The following theorem shows that the dimension-independent property still holds when private dataset has a subset classes of the pre-training dataset.

\begin{theorem}
\label{thm:error-zero-init-small}
   Let $\widehat{y}$ be a predictor trained by NoisyGD under the cross entropy loss with zero initialization.
   Assume that the training dataset is balanced.
   For multi-class classification problems under neural collapse with $K$ classes, subset of a gigantic dataset with $K_0 \geq K$ classes,  the misclassification error is 
   \begin{align*}
       \Pr[\widehat{y}\neq y] &\leq (K-1)\Phi\left(\frac{nC_{K,K_0}}{\sigma}\right) \leq (K-1) e^{-\frac{n^2C_{K,K_0}^2}{2\sigma^2}}
   \end{align*}
with $C_{K,K_0} =\frac{1}{K}\left[\frac{K\cdot K_0-2}{K^2(K_0-1)}\right]$ and $\sigma^2 = \frac{G^2}{2\rho}.$
\end{theorem}

\begin{proof}
Let
\begin{align*}
    M_0 &= \sqrt{\frac{K_0}{K_0-1}} P \left(I_{K_0} - \frac{1}{K_0}\mathbf{1}_{K_0} \mathbf{1}_{K_0}^T\right)=\sqrt{\frac{K_0}{K_0-1}} \left( P - \frac{1}{K_0}\sum_{k=1}^{K_0} P_k\mathbf{1}_{K_0}^T\right).
\end{align*}
Denote $M = \left[M_{1},\cdots, M_{K} \right]$ with each $M_{k}$ being a column of $M_0.$
Note that
\begin{align*}
    M_{k'}^T M_k = \frac{K_0}{K_0-1}\left(\delta_{k,k'} - \frac{1}{K_0}\right).
\end{align*}
We have
\begin{align*}
    \boldsymbol{\mu}_{n,K} :&= -\sum_{k'=1}^K n_{k'}\left(\frac{1}{K}\mathbf{1}_K - y^{k'}\right) M_{k'}^{T} M_{k}.
\end{align*}
For $j\neq k$, we have
\begin{align*}
    \left(\mu_{n,K}\right)_j = -\frac{n}{K}\left[\frac{1}{K} +\frac{K-1}{K(K_0-1)} -\frac{K-2}{K(K_0-1)}\right] = -\frac{n(K_0-2)}{K^2(K_0-1)}.
\end{align*}

For $j=k$, it holds

\begin{align*}
    \left(\mu_{n,K}\right)_j = -\frac{n}{K}\left[\frac{1}{K}-1  -\frac{K-1}{K(K_0-1)}\right] = \frac{n(K-1)K_0}{K^2(K_0-1)}.
\end{align*}
By the union bound, the misclassification error is
\begin{align*}
    (K-1)\Pr\left[ \mathcal{N}((\mu_{n,K})_k,\sigma^2) < \mathcal{N}((\mu_{n,K})_1,\sigma^2)\right] = (K-1)\Phi\left(\frac{nC_{K,K_0}}{\sigma}\right)
\end{align*}
with $C_{K,K_0} =\frac{1}{K}\left[\frac{K\cdot K_0-2}{K^2(K_0-1)}\right].$
\end{proof}

\section{Proofs of Section \ref{sec:remedy}}
\label{sec:proof-remedy}

\subsection{Details of the normalization}

Consider the case where the feature is shifted by a constant offset $v$.
The feature of the $k$-th class is $$\widetilde{x}_i = x_i - \frac{1}{n}\sum_{i=1}^n x_i = \widetilde{M}_k = M_k + v$$ with $M_k$ being the $k$-th column of the ETF $M$.

The offset $v$ can be canceled out by considering the differences between the features.
That is, we train with the feature $\widetilde{M}_k - \frac{1}{K}\sum_{j=1}^K \widetilde{M}_j$ for the $k$-th class.
In fact, let $P_k$ be the $k$-th column of $P$ and we have
\begin{align*}
    \widetilde{M}_k - \frac{1}{K}\sum_{j=1}^K \widetilde{M}_j &= M_k - \frac{1}{K}\sum_{j=1}^K M_j
    \\
    &=\sqrt{\frac{K}{K-1}} \left[ \left( P_k - \frac{1}{K}\sum_{i=1}^K P_i \right) - \frac{1}{K}\sum_{j=1}^K \left(P_j - \frac{1}{K}\sum_{i=1}^K P_i\right)\right]
    \\
    & = \sqrt{\frac{K}{K-1}} \left( P_k - \frac{1}{K}\sum_{j=1}^K P_j\right) = M_k.
\end{align*}

\subsection{Proof of  Theorem \ref{thm:pert}}

\begin{proof}[Proof of  Theorem \ref{thm:pert}]
Consider the case with $K=2$ and a projection vector $\widehat{P} = (e_1 + \Delta)$ with some perturbation $\Delta = (\Delta_1,\cdots,\Delta_p)$ such that $\|\Delta\|_{\infty} \leq \beta_0$ for some $0<\beta_0\ll p$.
$\widehat{P}$ can be generated by the pre-training dataset or the testing dataset.
Consider training with features $\widetilde{x}_i = \widehat{P}x_i.$
Then, the sensitivity of the NoisyGD is $G = \sup_{v}|\widehat{P}^T(e_1 + v)| = 1 + \beta + \beta|\Delta_1| + \beta(\sum_{j=1}^p |\Delta_j|) \leq 1 + \beta(1 + \beta_0 + p\beta_0).$
The output of Noisy-GD is then given by
\begin{align*}
    \widehat{\theta} = -\widehat{P}\cdot \left(\sum_{i=1}^n y_i\widetilde{x}_i \right) + \mathcal{N}(0,\sigma^2).
\end{align*}
Moreover, for any testing data point $e_1 + v$, define
\begin{align*}
   \widehat{\mu}_n = -\left(\sum_{i=1}^ny_i\widetilde{x}_i\right)\widehat{P}^T(e_1+ v) = \left( e_1 + V\right)^T \widehat{P}\widehat{P}^T (e_1 + v)
\end{align*}
with $V = \frac{1}{n}\sum_{i=1}^n v_i =: (V_1,\cdots, V_p).$

We now divide $\widehat{\mu}_n$ into four terms and bound each term separately.

For the first term $e_1^T\widehat{P}\widehat{P}^T e_1$, it holds
\begin{align*}
    e_1^T\widehat{P}\widehat{P}^T e_1 = (1 + e_1^T\Delta_1)^2 \leq (1 - \beta_0)^2.
\end{align*}

For the second term $V^T\widehat{P}\widehat{P}^T e_1$, we have
\begin{align*}
    V^T\widehat{P}\widehat{P}^T e_1  = (V_1+ V^T\Delta)(1 + \Delta_1)\leq |V_1+ V^T\Delta|(1 + \beta_0)
\end{align*}
Note that $V_1$ is the average of $n$ i.i.d.\ random variables bounded by $\beta$. By Hoeffding's inequality, we obtain
\begin{align*}
    |V_1| \leq \frac{\beta\log\frac{2}{\gamma}}{\sqrt{n}}, \hbox{ with probability at least } 1 - \gamma.
\end{align*}
Similarly, with confidence $1-\gamma$, it holds
\begin{align*}
    |V^T\Delta| \leq \frac{p\beta\beta_0\log\frac{2}{\gamma}}{\sqrt{n}}.
\end{align*}

The third term $e_1^T \widehat{P}\widehat{P}^T v$ can be bounded as 
\begin{align*}
    |e_1^T \widehat{P}\widehat{P}^T v|  = (1 + \Delta_1) \left(\sum_{j=1}^p v_i\left(1 + \Delta_i\right)\right)\leq (1 + \beta_0) \left(\beta + \beta_0\sqrt{p\log\frac{2}{\gamma}}\right),
\end{align*}
where the last inequality is a result of the Hoeffding's inequality by assuming that each coordinate of $v$ are independent of each others.
Moreover, without further assumptions on the independence of each coordinate of $v$, we have
\begin{align*}
     |e_1^T \widehat{P}\widehat{P}^T v|  = (1 + \Delta_1) \left(\sum_{j=1}^p v_i\left(1 + \Delta_i\right)\right)\leq (1 + \beta_0) \left(\beta + \beta_0p\right).
\end{align*}

Using the Hoeffding's inequality again, for the last term $V^T \widehat{P}\widehat{P}^T (e_1+v)$, it holds
\begin{align*}
    |V^T \widehat{P}\widehat{P}^T (e_1+v)| \leq \frac{(\beta + \beta_0\sqrt{p})(1 + \beta + \beta_0 + \beta\beta_0\sqrt{p})\log\frac{4}{\gamma}}{\sqrt{n}}
\end{align*}
with confidence $1 - \gamma$ if we assume that all coordinates of $v$ are independent of each other.
Without further assumptions on the independence of each coordinate of $v$, we have
\begin{align*}
    |V^T \widehat{P}\widehat{P}^T (e_1+v)| \leq \frac{(\beta + \beta_0p)(1 + \beta + \beta_0 + \beta\beta_0p)\log\frac{2}{\gamma}}{\sqrt{n}}.
\end{align*}

\end{proof}

\section{Some calculations on random Initialization}
\label{sec:ran-init}

In machine learning, training a deep neural network using (stochastic) gradient descent combined with random initialization is widely adopted \citep{DBLP:conf/icml/SutskeverMDH13}. The significance of random initialization on differential privacy in noisy gradient descent is also emphasized by \citep{ye2023initialization, wang2023unified}.
Extending our theory from zero initialization to random initialization is non-trivial and we discuss some details in this section.

\subsection{Gaussian initialization without offset}
For Gaussian initialization $\xi = (\xi_1,\cdots,\xi_p)\sim\mathcal{N}(0,I_p)$, we have
\begin{align*}
    \hat{\theta} &= \xi - \eta \left(\frac{n}{2} \cdot \frac{-e^{-\xi_1}}{1 + e^{-\xi_1}} \cdot (- \begin{bmatrix}
	-1\\
	0\\
	\vdots\\
	0\\
\end{bmatrix})  + \frac{n}{2} \cdot \frac{-e^{-\xi_1}}{1 + e^{-\xi_1}} \cdot \begin{bmatrix}
1\\
0\\
\vdots\\
0\\
\end{bmatrix} + \cN(0, \frac{G^2}{2\rho} I_p)\right) 
\\
&=\xi +\eta \left(\frac{e^{-\xi_1}}{1 + e^{-\xi_1}}\cdot\begin{bmatrix}
n \\
0\\
\vdots\\
0\\
\end{bmatrix}  + \cN(0, \frac{G^2}{2\rho} I_p)\right)
\end{align*}
The sensitivity is $\frac{e^{-\xi_1}}{1 + e^{-\xi_1}}<1$. Consider $x = (-1,0,\cdots,0)^T$. We have
\begin{align*}
    \widehat{\theta}^Tx = -\xi_1 + \eta\left( -\frac{ne^{-\xi_1}}{1 + e^{-\xi_1}}\right) + \mathcal{N}(0,\frac{G^2}{2\rho})=: \mu_{\xi_1,n} + \mathcal{N}(0,\frac{G^2}{2\rho}).
\end{align*}
The misclassification error is 
\begin{align*}
    \Pr[\widehat{\theta}^Tx>0] &= \mathbb{E}_{\xi_1\sim\mathcal{N}(0,1)}\Pr\left[\left.\mathcal{N}\left(\mu_{\xi_1,n},\frac{G^2}{2\rho}\right)>0\right|\xi_1\right] 
    \\
    &= \mathbb{E}_{\xi_1\sim\mathcal{N}(0,1)}\left[\Phi\left(\frac{\sqrt{2\rho}\mu_{\xi_1,n}}{G}\right)\right]
\end{align*}

\subsection{Gaussian initialization with off-set}
Denote $x_1 = -e_1 + v$ and $x_2 = e_1 + v$ with $\|v\|_{\infty} \leq \beta.$
For the logistic loss $\ell(y,\theta^Tx) = \log(1 + e^{-y\theta^T x}),$ we have
\begin{align*}
    g(\theta,y\cdot x) := \nabla_{\theta}\ell(y,\theta^T x) = \frac{e^{-y\theta^Tx}}{1 + e^{-y\theta^T x}} (-yx).
\end{align*}
Denote
\begin{align*}
    g_1(\theta) = g(\theta,-1\cdot x_1) = \frac{e^{\theta^T x_1}}{1 + e^{\theta^T x_1}} x_1
\end{align*}
and
\begin{align*}
    g_2(\theta) = g(\theta,1\cdot x_2) = \frac{e^{-\theta^T x_2}}{1 + e^{-\theta^T x_2}} (-x_2).
\end{align*}

If we shift the feature by some vector $v$, then the loss function is
\begin{align*}
    R_n = \frac{n}{2}\log(1 + e^{\theta^T x_1}) + \frac{n}{2}\log(1 + e^{-\theta^T x_2}). 
\end{align*}

And the gradient is
\begin{align*}
    \nabla_\theta R_n(\theta) &= \frac{n}{2} \left( g_1(\theta) + g_2(\theta)\right).
\end{align*}
Thus, the output of one-step NoiseGD is given by
\begin{align*}
    \widehat{\theta} = \theta_0 - \frac{\eta n}{2}\left[g_1(\theta_0) + g_2(\theta_0) + \mathcal{N}(0,\sigma^2) \right].
\end{align*}

Let $\mu_{\xi} = \xi - \frac{\eta n}{2} \left[g_1(\xi) + g_2(\xi) \right].$ Then, we have
\begin{align*}
    \mu_{\xi}^T e_1 = \xi_1 - \frac{\eta n e^{\xi^T x_1}}{2 + 2 e^{\xi^T x_1}}(-1 + v_1) + \frac{\eta n e^{\xi^T x_2}}{2 + 2 e^{\xi^T x_2}}(1 + v_1).
\end{align*}
And the misclassification error is 
\begin{align*}
    \mathbb{E}_{\xi} \left( \Phi\left(-\frac{\sqrt{2\rho}\mu_{\xi}^{T}e_1}{G}\right)\right).
\end{align*}

\end{document}